%% file: ijcai25.tex
\pdfoutput=1
\documentclass{article}
\usepackage{ijcai25}

\usepackage{times}
\usepackage{soul}
\usepackage{url}
\usepackage[utf8]{inputenc}
\usepackage[small]{caption}
\usepackage{graphicx}
\usepackage{amsmath}
\usepackage{amsthm}
\usepackage{booktabs}
\usepackage{algorithm}
\usepackage{algorithmic}
\usepackage[switch]{lineno}

\usepackage{wrapfig}

\usepackage{subfig}
\usepackage{booktabs}
\usepackage{amsmath,amssymb,amsfonts}
\usepackage{amsthm}
\usepackage{float}
\usepackage{longtable}

\newcommand{\stitle}[1]{\vspace*{0.4em}\noindent{\bf #1.\/}}

\usepackage[hidelinks]{hyperref}

\newtheorem{theorem}{Theorem}

\title{Efficient Diversity-based Experience Replay for Deep Reinforcement Learning}

\author{
Kaiyan Zhao\textsuperscript{1,2}\thanks{Authors contributed equally}\and
Yiming Wang\textsuperscript{2$\ast$}\and
Yuyang Chen\textsuperscript{1}\and 
Yan Li\textsuperscript{3}\and
Leong Hou U\textsuperscript{2}\and
Xiaoguang Niu\textsuperscript{1}\thanks{Corresponding author}\\
\affiliations
\textsuperscript{1}School of Computer Science, Wuhan University, Wuhan, China\\
\textsuperscript{2}State Key Laboratory of Internet of Things for Smart City, University of Macau, Macao, China\\
\textsuperscript{3}School of Artificial Intelligence Shenzhen Polytechnic University, China\\
\emails
\{zhao.kaiyan, xgniu\}@whu.edu.cn,wang.yiming@connect.um.edu.mo,yb57411@szpu.edu.cn
\\chenyuyang0520@gmail.com,ryanlhu@um.edu.mo
}

\begin{document}

\maketitle

\input{content/0_abstract}
\input{content/1_introduction}
\input{content/3_background}
\input{content/4_methodology}
\input{content/5_experiment}
\input{content/2_related_work}

\input{content/6_conclusion}

\section*{Acknowledgments}

This work was supported in part by the Key Research and Development Project of Hubei Province (2022BCA057), the Science and Technology Development Fund Macau SAR (0003/2023/RIC, 0052/2023/RIA1, 0031/2022/A, 001/2024/SKL for SKL-IOTSC), the Shenzhen-Hong Kong-Macau Science and Technology Program Category C (SGDX20230821095159012), the National Natural Science Foundation of China (62402325), and the Research Foundation of Shenzhen Polytechnic University (6022310014K, 6022312054K). This work was performed in part on the supercomputing system at the Supercomputing Center of Wuhan University and in part at SICC which is supported by SKL-IOTSC, University of Macau.

\bibliographystyle{named}
\bibliography{ijcai25}

\newpage
\setcounter{theorem}{0}
\setcounter{proposition}{0}
\appendix
\onecolumn
\input{content/7_appendix}
\onecolumn
\end{document}

%% file: content/0_abstract.tex
\begin{abstract}
Experience replay is widely used to improve learning efficiency in reinforcement learning by leveraging past experiences. 
However, existing experience replay methods, whether based on uniform or prioritized sampling, often suffer from \textit{low efficiency}, particularly in real-world scenarios with \textit{high-dimensional state spaces}. 
To address this limitation, we propose a novel approach, Efficient Diversity-based Experience Replay (EDER). EDER employs a determinantal point process to model the diversity between samples and prioritizes replay based on the diversity between samples. 
To further enhance learning efficiency, we incorporate Cholesky decomposition for handling large state spaces in realistic environments. Additionally, rejection sampling is applied to select samples with higher diversity, thereby improving overall learning efficacy.
Extensive experiments are conducted on robotic manipulation tasks in MuJoCo, Atari games, and realistic indoor environments in Habitat. The results demonstrate that our approach not only significantly improves learning efficiency but also achieves superior performance in high-dimensional, realistic environments.
\end{abstract}

%% file: content/1_introduction.tex
\section{Introduction}

In recent years, Deep Reinforcement Learning~\cite{franccois2018introduction,yiming1} has surged in popularity, achieving remarkable success in complex decision-making tasks. DRL has been successfully applied to games~\cite{schrittwieser2020mastering,silver2017mastering}, robotic control~\cite{andrychowicz2020learning,levine2016end}, autonomous driving scenarios including traffic light control~\cite{yangming1,yangming3,yangming2}, and other domains, demonstrating its powerful learning and decision-making capabilities.

\begin{figure}[!htb]
    \centering
    \includegraphics[scale=1.45]{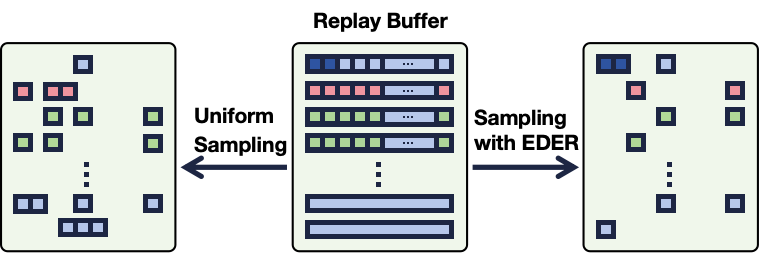}
    \caption{Sample distribution comparison of the replay buffer. Left: Uniform sampling results in an imbalanced distribution, with some data types overrepresented and others underrepresented. Right: Our method achieves a more balanced and diverse selection of samples, enhancing overall diversity and improving learning efficiency.}
    \label{fig:confounder_illustration}
    \vspace{-10pt}
\end{figure}

However, DRL still faces significant challenges in practical applications, particularly in handling sparse reward signals~\cite{hare2019dealing}, high-dimensional state spaces~\cite{ibrahimi2012efficient}, and low sample efficiency~\cite{yarats2021improving,yiming2}. Sparse reward signals make it difficult for agents to learn effective policies from limited positive feedback, resulting in slow and inefficient learning processes. Additionally, high-dimensional state spaces further complicate the learning process and increase computational burdens, making existing methods inefficient in large-scale and complex environments.

To address these issues, Experience Replay (ER) has been widely adopted as a key mechanism. ER improves sample efficiency and stabilizes the learning process by storing the agents' past experiences and randomly sampling them for training. Despite the improvements ER offers in sample efficiency, existing methods still suffer from inefficiency and suboptimal performance in high-dimensional state spaces. Recent studies~\cite{andrychowicz2020learning,levine2016end,todorov2012mujoco,jiang2024importance,zhao2018energy,fang2019curriculum} have focused on enhancing ER’s sampling strategies to improve their applicability and efficiency in complex environments. For instance, Hindsight Experience Replay (HER)~\cite{andrychowicz2017hindsight} generates more positive feedback samples to enhance learning efficiency; Prioritized Experience Replay (PER)~\cite{schaul2015prioritized} assigns priorities to samples based on their temporal difference (TD) errors; and Topological Experience Replay (TER)~\cite{hong2022topological} builds a trajectory graph and performs breadth-first updates from terminal states. Large Batch Experience Replay (LaBER)~\cite{pmlr-v162-lahire22a} improves sample efficiency by sampling large batches and performing focused updates. The Reducible Loss (ReLo) method~\cite{sujit2023prioritizing} ranks samples based on their learnability, measured by consistent loss reduction. However, these approaches generally struggle to efficiently select valuable samples in high-dimensional state spaces, leading to persistent issues of low efficiency and high-dimensional state space challenges in DRL.

To tackle these challenges, we propose a novel Experience Replay framework, Efficient Diversity-based Experience Replay (EDER). EDER utilizes Determinantal Point Processes (DPP)~\cite{kulesza2012determinantal} to model the diversity among samples and determines replay priorities based on this diversity, effectively avoiding the redundant sampling of ineffective data points. Furthermore, to handle high-dimensional state spaces in real-world environments, EDER employs Cholesky decomposition~\cite{krishnamoorthy2013matrixinversionusingcholesky}, significantly reducing computational complexity. Combined with rejection sampling techniques~\cite{neal2003slice,azadi2018discriminator}, EDER selects samples with higher diversity for training, thereby further enhancing overall learning efficiency.

Our main contributions are as follows. Firstly, we propose the Efficient Diversity-based Experience Replay (EDER) framework, which prioritizes sample diversity and significantly enhances experience replay (ER) efficiency, especially in high-dimensional state spaces and environments with sparse rewards. Secondly, we introduce Cholesky decomposition and rejection sampling to effectively address computational bottlenecks in large state spaces and optimize the ER mechanism by selecting more diverse samples. Lastly, we conduct extensive experimental validations across multiple complex environments, including Habitat~\cite{habitat19iccv,puig2023habitat3}, Atari games~\cite{mnih2013playingatarideepreinforcement}, and MuJoCo~\cite{todorov2012mujoco}. The results demonstrate that EDER not only significantly improves learning efficiency but also achieves superior performance in high-dimensional, realistic environments, thereby validating its effectiveness and adaptability in various complex settings.

\begin{figure*}[!htb]
\vspace{-15pt}
\subfloat{\includegraphics[width=1\textwidth]{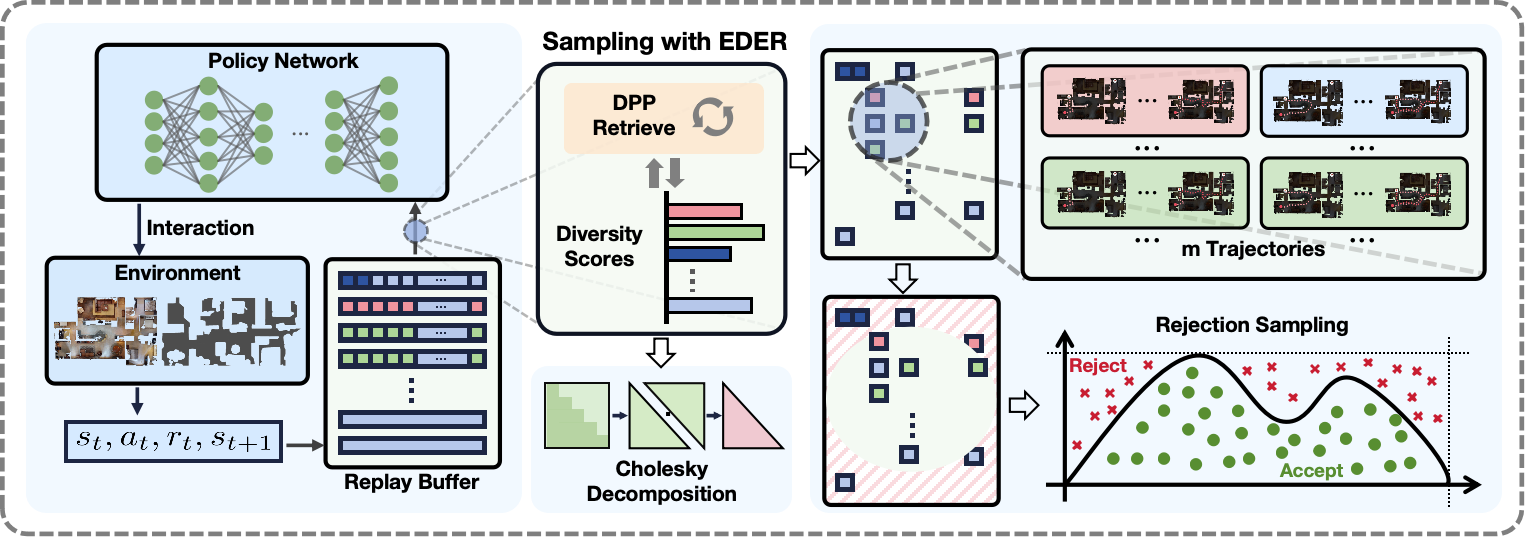}}
\caption{In the EDER framework, we leverage the Determinantal Point Process (DPP) to compute diversity scores for trajectories via Cholesky decomposition, enhancing the sampling process. Specifically, our method first uses these diversity scores to select the top 
$m$ most diverse trajectories. Next, we apply a rejection sampling technique to choose a subset of these trajectories for policy updates. The resulting diverse samples facilitate more efficient learning, particularly in high-dimensional environments.}
\vspace{-10pt}
\label{fig:framework}
\end{figure*}

%% file: content/3_background.tex
\section{Preliminaries}

\stitle{Reinforcement Learning}
Reinforcement Learning (RL) is a learning paradigm where agents autonomously learn to make sequential decisions by interacting with an environment, with the goal of maximizing cumulative rewards. The problem is typically formalized as a Markov Decision Process (MDP), which is defined by a tuple \( \langle S, A, P, R, \gamma \rangle \), where \( S \) represents the state space, \( A \) represents the action space, \( P \) defines the state transition probabilities, \( R \) denotes the reward function, and \( \gamma \) is the discount factor. At each discrete time step \( t \), the environment is in a state \( s_t \), and the agent selects an action \( a_t \) according to a policy \( \pi \). The environment then transitions to a new state \( s_{t+1} \) based on the transition probability \( P(s_{t+1} \mid s_t, a_t) \), and the agent receives a scalar reward \( r_{t+1} \). The agent's objective is to learn an optimal policy \( \pi^* \) that maximizes the expected cumulative discounted reward starting from any initial state \( s_t \):
\[
V^\pi(s_t) = \mathbb{E}\left[\sum_{k=0}^{\infty} \gamma^k r_{t+k+1} \mid s_t=s, \pi \right],
\]
where \( V^\pi(s_t) \) is the value function that estimates the expected return when following policy \( \pi \) from state \( s_t \).

\stitle{Experience Replay}
Experience replay is essential in deep reinforcement learning, enabling agents to store and revisit past experiences via a replay buffer. This mechanism mitigates the issue of correlated data in online learning and improves sample efficiency. Two prominent techniques that enhance experience replay are Prioritized Experience Replay (PER) and Hindsight Experience Replay (HER):
PER improves replay efficiency by prioritizing experiences based on their learning value, typically measured by the temporal difference (TD) error \( \delta_t = r_{t+1} + \gamma V(s_{t+1}) - V(s_t) \), where \( \gamma \) is the discount factor and \( V(s_t) \) is the value function of state \( s_t \).In PER, an experience is assigned a priority \( p_t = |\delta_t| + \epsilon \), where \( \epsilon \) ensures non-zero priority. The probability \( P(i) \) of sampling an experience is proportional to its priority:
\[
P(i) = \frac{p_i^\alpha}{\sum_k p_k^\alpha},
\]
where \( \alpha \) controls the degree of prioritization. By focusing on experiences with higher TD errors, PER enhances learning efficiency and accelerates convergence. HER addresses the challenge of sparse rewards by augmenting the replay buffer with re-labeled experiences, wherein failed attempts are reinterpreted as successes for alternative goals. Specifically, if the agent fails to achieve the intended goal \( g \) at state \( s_t \), HER re-labels this experience as successful for a new goal \( g' \), such as a subsequent state \( s_{t+k} \). The re-labeled reward function is defined as follows:

\[
r_{t+1} = 
\begin{cases} 
1 & \text{if } s_{t+k} = g', \\
0 & \text{otherwise}.
\end{cases}
\]
This approach increases the number of successful experiences, thereby enhancing learning efficiency in environments with sparse rewards by effectively increasing the density of positive samples.

\stitle{Determinantal Point Processes}
Determinantal Point Processes (DPPs) are widely used probabilistic models that capture diversity within a set of points. For a discrete set \( Y = \{x_1, x_2, \dots, x_N\} \), a DPP defines a probability measure over all possible subsets of \( Y \), where the probability of selecting a subset \( Y \subseteq Y \) is proportional to the determinant of a positive semi-definite kernel matrix \( L \) corresponding to \( Y \). Specifically, the probability of sampling a subset \( Y \) is:
\[
P(Y) = \frac{\text{det}(L_Y)}{\text{det}(L + I)},
\]
where \( L_Y \) is the principal submatrix of \( L \) indexed by the elements in \( Y \), and \( I \) is the identity matrix. The determinant \( \text{det}(L_Y) \) measures the diversity of \( Y \) by the volume spanned by the vectors associated with \( Y \).
In practice, the kernel matrix \( L \) is often the Gram matrix \( L = X^T X \), where each column of \( X \) represents a feature vector of an element in \( Y \). The geometric interpretation of DPPs implies that subsets with more orthogonal feature vectors—indicating higher diversity—are more likely to be selected. This makes DPP effective for sampling diverse trajectories and goals in reinforcement learning, where diversity in the experience buffer is crucial for robust learning.

%% file: content/4_methodology.tex
\section{Methodology}

In this study, we propose a novel approach named Efficient Diversity-based Experience Replay (EDER), which enhances exploration and sample efficiency in reinforcement learning (RL) through a diversity-based trajectory selection module, which selects transitions from each trajectory based on their diversity rankings. The EDER algorithm leverages Determinantal Point Processes (DPPs) to evaluate the diversity of trajectories, enabling the exploration of a broader range of informative data. Following exploration, high-quality data is replayed to improve training efficiency. Furthermore, we employ Cholesky decomposition and rejection sampling to enhance computational efficiency, particularly in realistic environments with high-dimensional state spaces.

\stitle{Data Preprocessing}
We define the state transition dataset $T$  as a collection of state transitions accumulated during the agent's interaction with the environment, represented as: $T = \left\{ \{s_0, s_1\}, \{s_2, s_3\}, \ldots, \{s_{T-1}, s_T\} \right\} $ where each element $\{s_i, s_{i+1}\}$ represents a transition from state $s_i$ to state $ s_{i+1} $. 
In our framework, we partition $ T $ into multiple partial trajectories of length $ b $, denoted as $ \tau_j $, each covering a state transition from $ t=js $ to $ t=js+b-1 $, where $ s $ represents the sliding step length. The trajectories are quantified by sliding the window of length $ s = b $, where the meticulous segmentation allows us to analyze and understand the behavioral patterns of intelligent agents at different stages. The specific formula is as follows:
\begin{equation}
T = \left\{ \tau_j = \{s_{jb}, s_{jb+1}, \ldots, s_{jb+b-1}\} \mid j = 0, 1, \ldots,  \frac{T}{b} - 1 \right\}
\end{equation}
Here, $ \tau_j $ denotes the partial trajectory of group $ j $ covering the state transition from $ s_{jb} $ to $ s_{jb+b-1} $. Each $ \tau_j $ is a sliding window of length $ b $, demonstrating the behavior of the agent and its environmental adaptation during that time period.

\subsection{Diversity-Based Trajectory Selection Module}

The objective of this module is to select diverse trajectories from the replay buffer, enhancing learning by utilizing a wide range of experiences. A set of summary timelines describing the key trajectory events is generated from the entire collection of trajectories, which involves the following steps:

\stitle{Trajectory Segmentation} The entire sequence of state transitions during an interaction, denoted as $\tau$, is segmented into several partial trajectories $\tau_j$ of length $b$. Each segment $\tau_j$ covers transitions from state $s_n$ to $s_{n+b-1}$, allowing for detailed capture of dynamics between state transitions. For clarity, we set a sliding window of $b = 2$ in this part, while other values are explored in the ablation studies. Under this setting, a trajectory $\tau$ can be divided into $N_p$ partial segments.

\vspace{-10pt}
\[
\tau= \bigl\{
\{\underbrace{s_0, s_1}_{\tau_1}\}, 
\{\underbrace{s_2, s_3}_{\tau_2}\},
\{\underbrace{s_4, s_5}_{\tau_3}\},
\ldots, 
\{\underbrace{s_{T-1}, s_T}_{\tau_{N_p}}\}
\bigr\}
\]
\vspace{-10pt}

\stitle{Diversity Assessment}To effectively evaluate the diversity of each partial trajectory $\tau_j$, we adopt the theoretical framework of Determinantal Point Processes(DPPs). Specifically, the diversity metric $d_{\tau_j}$ for a partial trajectory $\tau_j$ is defined as the determinant of its corresponding kernel matrix:

\begin{equation}
  d_{\tau_j} = \det(L_{\tau_j})
  \label{eq:d_tau1}
\end{equation}
Intuitively, the determinant quantifies the $n$-dimensional volume spanned by the embedded state transitions in $\tau_j$, assigning higher values to sets of transitions that are more linearly orthogonal and thus more diverse.
Here, $L_{\tau_j}$ is the kernel matrix constructed from the state transitions within trajectory $\tau_j$, defined as:

\begin{equation}
  L_{\tau_j} = M^T M
  \label{eq:L_tauj}
\end{equation}
The columns of matrix $M$ are the $\ell_2$-normalized vector representations $\hat{s}$ of each state $s$ in trajectory $\tau_j$.
\begin{theorem}[Correlation between Determinant and Diversity]
Let $M \in \mathbb{R}^{d \times b}$ be a matrix whose columns are the $\ell_2$-normalized state vectors $\hat{s}$ in trajectory $\tau_j$. The determinant $\det(L_{\tau_j})$ of the kernel matrix $L_{\tau_j} = M^T M$ reaches its maximum value when the state vectors are mutually orthogonal, indicating the highest diversity of the trajectory.
\label{thm:det_diversity}
\end{theorem}
Proof in Appendix~\ref{sec:proofs}. The choice of Determinantal Point Processes is motivated by the ability of $\det(L_{\tau_j})$ to effectively measure the diversity of state vectors within trajectory $\tau_j$. Based on Theorem \ref{thm:det_diversity}, a larger determinant indicates higher diversity of the trajectory.

The determinant $\det(L_{\tau_j}) = \det(M^T M)$ is equal to the square of the volume of the parallelepiped spanned by the columns of matrix $M$. When the vectors are mutually orthogonal, the volume and thus the determinant reaches its maximum value, reflecting the highest independence and diversity of the state vectors. Conversely, if the vectors are linearly dependent, both the volume and the determinant decrease, indicating reduced diversity. Additionally, in DPPs, the kernel matrix $L_{\tau_j}$ captures the similarities between state vectors, inherently favoring the selection of diverse and minimally redundant subsets. Therefore, DPPs are an ideal choice for evaluating the diversity of trajectories in reinforcement learning \cite{kunaver2017diversity}. A larger $d_{\tau_j}$ indicates that the state vectors are more uniformly distributed in the feature space with lower similarity, reflecting higher diversity. This is crucial for policy training in reinforcement learning, as diversified data facilitates better policy generalization and adaptation to various environmental conditions.

\stitle{Sampling Strategy} The total diversity of a trajectory $\tau$, denoted as $d_{\tau}$, is defined as the sum of the diversities of all its constituent partial trajectories:
\begin{equation}
d_{\tau} = \sum_{j=1}^{N_p} d_{\tau_j}
\label{eq:d_tau2}
\end{equation}
Equation~(\ref{eq:d_tau2}) provides a comprehensive measure, effectively reflecting the overall diversity of the trajectory.
We employ a non-uniform sampling strategy to prioritize trajectories with higher diversity:
\begin{equation}
p(\tau_i) = \frac{d_{\tau_i}}{\sum_{n=1}^{N_e} d_{\tau_n}},
\label{eq:tau_i}
\end{equation}
where $N_e$ is the total number of trajectories in the replay buffer, this strategy enhances learning efficiency by increasing the likelihood of selecting highly diverse trajectories, thereby enabling the agent to effectively learn and adapt to various environmental conditions.

Although the determinant effectively measures diversity, its direct computation in high-dimensional state spaces is \textit{\textbf{computationally intensive}}, especially for large trajectory lengths $b$. Therefore, we employ Cholesky decomposition and rejection sampling to optimize computation speed in the following section. The diversity metric $d_{\tau_j}$ quantifies the independence and diversity of state vectors within trajectory $\tau_j$ using the determinant.

\subsection{Improving Computational Efficiency}
Scaling to high-dimensional environments is crucial for the applicability of deep reinforcement learning algorithms. Traditional approaches often fail due to computational inefficiency, especially when dealing with large state spaces where calculations become difficult and time-consuming. Computing Determinantal Point Processes (DPPs) in high-dimensional state spaces is computationally intensive due to the \textbf{\textit{complexity of calculating}} large kernel matrices. This challenge is particularly acute in extensive state spaces where traditional methods struggle to maintain efficiency. To address this issue, we propose an optimized approach that integrates Cholesky decomposition and rejection sampling into our method. This approach reduces computational costs while preserving the effectiveness of DPPs with theoretical guarantees, making them applicable to complex reinforcement learning scenarios.

\stitle{Cholesky Decomposition}
To simplify the determinant calculation of the kernel matrix, a key operation in DPP, we employ Cholesky decomposition.  For a window length $ b $, given state vectors $\hat{s}_1, \hat{s}_2, \ldots, \hat{s}_b$, we construct the matrix $ M $ as $M = [\hat{s}_1, \hat{s}_2, \ldots, \hat{s}_b]$.
The kernel matrix $ L_{\tau_j} $ is then formed as Equation~(\ref{eq:L_tauj}). To efficiently compute the determinant of $ L_{\tau_j} $, we apply Cholesky decomposition, which decomposes $ L_{\tau_j} $ into a product of a lower triangular matrix $ L_C $ and its transpose $ L_C^T $:

\begin{equation}
L_{\tau_j} = L_C L_C^T
\end{equation}
The determinant is then computed as the product of the squares of the diagonal elements of $ L_C $:

\begin{equation}
\det(L_{\tau_j}) = \prod_{i=1}^b l_{ii}^2
\label{eq:lii}
\end{equation}
Here, $l_{ii}$ denotes the $i$-th diagonal element. 
With Cholesky decomposition, the time complexity of determinant computation for each segment is reduced from $O(n^3)$ to $O(n^2)$.

\stitle{Rejection Sampling} Numerous trajectories are utilized in training, resulting in high sampling inefficiency. To enhance \textbf{\textit{sampling efficiency}}, we introduce rejection sampling. This method effectively filters trajectory segments \emph{before insertion into the replay buffer}, which is particularly useful in high-dimensional state spaces where storing redundant segments incurs significant computational overhead. By prioritizing segments with higher diversity scores, rejection sampling minimizes the retention of less informative segments. Consequently, computational resources are focused on the most diverse and relevant experiences, ensuring that the replay buffer contains the most valuable transitions.

The rejection sampling process is detailed as follows:
First, for each trajectory segment \( \tau_j \), we compute its diversity score \( Q_j \) using Equation~(\ref{eq:d_tau1}) and (\ref{eq:lii}):
\begin{equation}
    Q_j = d_{\tau_j} = \det(L_{\tau_j}) = \prod_{i=1}^b l_{ii}^2
\end{equation}
Next, we determine a normalization constant \( M \), defined as:
\begin{equation}
    M = \max(Q)
\end{equation}
This ensures that for all trajectory segments \( \tau_j \), the acceptance probability \( \alpha = \frac{Q_j}{M} \) remains in the valid range \([0, 1]\).
During the rejection sampling process, we uniformly select a candidate segment \( \tau' \) from the current batch of generated segments. Then we draw a uniformly distributed random number \( u \sim U(0, 1) \). If the sampled number satisfies:
\begin{equation}
    u \leq \alpha = \frac{d_{\tau'}}{M},
\end{equation}
we accept the candidate trajectory segment \( \tau' \); otherwise, we reject it and resample. This process yields a set of diverse trajectory segments, which are then inserted into the replay buffer. Subsequently, training batches are sampled from the buffer according to Equation~(\ref{eq:tau_i}).

\subsection{Time Complexity Analysis}
\begin{theorem}
The time complexity of the EDER algorithm is \( O(N b d + N b^3 + N \log m + m) \) without employing Cholesky decomposition and rejection sampling, and it is reduced to \( O(N b d + N b^2 + N \log m + m) \) after integrating these optimizations. Here, \( N \) denotes the number of state transitions, \( b \) the segment length, \( d \) the dimensionality of the state vectors, and \( m \) the number of top trajectories selected.
\label{thm:2}
\end{theorem}

The proof is in Appendix~\ref{sec:proofs}. Based on Theorem~\ref{thm:2}, the integration of Cholesky decomposition and rejection sampling significantly reduces the overall computational complexity of the EDER algorithm, especially when dealing with large segment lengths b. This improvement makes the EDER algorithm more efficient and scalable for high-dimensional reinforcement learning tasks, enhancing its applicability in complex environments.

\begin{algorithm}[ht]
\caption{EDER}
\label{algorithm:EDER}
\begin{algorithmic}[1]
    \STATE \textbf{Initialize:} Replay buffer $ \mathcal{D} $, diversity score list $ Q $, segment length $ b $
    \WHILE{not converged}
        \STATE Initialize state $ s_0 $
        \FOR{$ t = 1 $ to $ T $}
            \STATE Select action $ a_t $ via policy $ \pi(s_t, \theta) $
            \STATE Execute $ a_t $, observe $ s_{t+1} $, receive $ r_t $
            \STATE Store $ (s_t, a_t, r_t, s_{t+1}) $ in $ \mathcal{D} $
        \ENDFOR
        \FOR{each trajectory $ \tau $ in $ \mathcal{D} $}
            \STATE Segment $ \tau $ into sub-trajectories $ \tau_j $
            \STATE Compute diversity score $Q_j$ using $\det(L_{\tau_j})$ via Cholesky Decomposition \hfill $\triangleright$ Equation~(\ref{eq:lii})
            \STATE Append $ Q_j $ to $ Q $
        \ENDFOR
        \STATE Set $ M = \max(Q) $
        \FOR{each $ Q_j $ in $ Q $}
            \STATE Calculate acceptance $ \alpha = \frac{Q_j}{M} $
            \STATE Accept corresponding $ \tau_j $ if $ u \leq \alpha $, else discard
        \ENDFOR
        \STATE Sample $ \mathcal{B} \sim \mathcal{D} $ using Eq.~(\ref{eq:tau_i})
        \STATE Update $ \theta $ using $ \mathcal{B} $
    \ENDWHILE
\end{algorithmic}
\end{algorithm}
\vspace{-10pt}

%% file: content/5_experiment.tex
\section{Experiments}

Our experiments aim to rigorously evaluate the performance of the proposed Efficient Diversity-based Experience Replay (EDER) method across multiple environments, focusing on its effectiveness compared to established baseline methods. The experiments are conducted in Mujoco, Atari, and real-life Habitat environments, each selected to highlight different aspects of EDER's capabilities. Detailed environment settings are provided in Appendix~\ref{sec:ex_imple_detail}. Details are available at \href{https://arxiv.org/abs/2410.20487}{https://arxiv.org/abs/2410.20487}.

\stitle{Baselines}
We compare our method against the following baselines. DDPG~\cite{lillicrap2019continuouscontroldeepreinforcement}: a deep reinforcement learning algorithm for continuous action spaces, combining deterministic policy gradients with Q-learning.
DQN~\cite{mnih2013playingatarideepreinforcement}: a widely used algorithm for discrete action spaces, approximating the Q-value function with deep neural networks.
HER~\cite{andrychowicz2017hindsight}: Hindsight Experience Replay enables learning from alternative goals that could have been achieved, improving efficiency in sparse reward settings.
PER~\cite{schaul2015prioritized}: Prioritized Experience Replay enhances learning by prioritizing important transitions.
TER ~\cite{hong2022topological}: Topological Experience Replay builds a graph from experience trajectories to track predecessors, then performs breadth-first updates from terminal states like reverse topological sorting.
LaBER ~\cite{pmlr-v162-lahire22a}: Large Batch Experience Replay samples a large batch from the replay buffer, computes gradient norms, downsamples to a smaller batch based on priority, and uses this mini-batch to update the policy.
Relo ~\cite{sujit2023prioritizing}: Reducible Loss (ReLo) is a sample prioritization method that ranks samples by their learnability, measured by the consistent reduction in their loss over time.

\begin{table}[H]
\centering
\setlength{\tabcolsep}{3.8pt}
\small 
\begin{tabular}{lrrr}
\toprule
\textbf{Methods} & \textbf{Residential} & \textbf{Office} & \textbf{Commercial} \\
\midrule
DDPG          & 9.0 ± 2.5   & 27.5 ± 1.9   & 23.0 ± 2.0   \\
DDPG+HER      & 35.0 ± 2.8   & 42.5 ± 2.1   & 42.0 ± 2.3   \\
DDPG+PER      & 23.0 ± 3.0   & 45.0 ± 2.3   & 34.5 ± 2.4   \\
DDPG+TER      & 48.0 ± 2.1   & 47.0 ± 1.5   & 54.5 ± 3.5   \\
DDPG+LaBER    & 52.0 ± 3.0   & 54.0 ± 2.3   & 48.5 ± 2.1   \\
DDPG+Relo     & 55.0 ± 1.3   & 53.0 ± 2.9   & 59.5 ± 1.8   \\
DDPG+EDER   w/o R.S.   & 58.0 ± 3.1   & 50.8 ± 2.1   & 55.5 ± 3.3   \\
DDPG+EDER   w/o C.D.    & 60.1 ± 2.6   & 68.0 ± 3.9   & 58.5 ± 2.7   \\
DDPG+EDER     & \textbf{64.0 ± 3.3} & \textbf{74.5 ± 2.4} & \textbf{68.4 ± 2.5}   \\
\bottomrule
\vspace{-15pt} 
\end{tabular}
\caption{Success rates (\%) across environments in HM3D.}
\vspace{-10pt}
\label{tab:habitat_tasks}
\end{table}

\begin{table*}[ht]
\setlength{\tabcolsep}{5pt}
\centering
\small
\vspace{-5pt}
\begin{tabular}{lrrrrrrrrrrr}
\toprule
\textbf{Method} & \textbf{Alien} & \textbf{Asterix} & \textbf{BeamR.} & \textbf{Breakout} & \textbf{CrazyCli.} & \textbf{Demo.} & \textbf{Krull} & \textbf{KungFu.} & \textbf{MsPac.} \\
\midrule
Random  &   227.8      & 210.0   & 363.9     & 1.7  & 10,780      & 152.1    & 1,598.3 & 258.5 & 307.3  \\
DQN     &  3,069.0     & 6,012.0 & 6,846.0   & 401.2 & 14,103.3  & 9,711.0   & 3,805.2 & 23,270.3 & 2,311.0  \\
DQN+PER &   4,204.2    & 31,527.3 & 23,384.0  & 373.9 & 141,161.0 & 71,846.7 & 9,728.6 & 39,581.2 & 6,519.1\\
DQN+TER   &  4,298.5   & 24,798.5 & 24,432.1  & 420.3 & 142,321.5 & 73,346.2 & 9,643.1 & 39,832.9 & 6,587.0 \\
DQN+LaBER & 4,365.2    & 39,172.1 & 23,543.4  & 462.2 & 145,672.2 & 75,128.0 & 9,764.7 & 41,823.0 & 6,691.4 \\
DQN+Relo  &  4,312.9   & 38,432.4 & 26,064.0  & 492.5 & 144,875.0 & 75,442.1 & 9,734.4 & 41,232.0 & 6,613.1 & \\
\midrule
DQN+EDER  w/o R.S. &  4,292.9  & 44,823.9 & 25,032.0 & 438.8 & 140,274.0 & 74,924.4 & 9,374.4 & 40,387.0 & 6,124.1 & \\
DQN+EDER  w/o C.D. &  4,689.4  & 50,283.7 & 25,731.0 & 481.9 & 142,328.6 & 75,326.1 & 9,353.4 & 40,983.0 & 6,493.1 & \\
DQN+EDER & \textbf{4,723.1} & \textbf{54,328.5} & \textbf{26,543.0} & \textbf{516.0} & \textbf{147,305.0} & \textbf{76,150.1} & \textbf{9,805.0} & \textbf{43,310.0} & \textbf{6,722.2} &\\

\midrule
\midrule

\textbf{Method} & \textbf{Enduro}  & \textbf{Frost.} & \textbf{HERO} & \textbf{Jamesb.} & \textbf{Kangar.} & \textbf{Pong}& \textbf{Qbert} & \textbf{River.} & \textbf{ZaxxPH.} \\
\midrule
Random  &  0.0      & 65.2    &  1,027.0  &  29.0   & 52.0     & -20.7 & 163.9    &1,338.5    &  32.5           \\
DQN     &  301.8   & 328.3   &  19,950.0 &  576.7  & 6,740.0  &  18.9 & 10,596.0 & 8,316.0    & 4,977.0      \\
DQN+PER &  2,093.0  & 4,380.1 &  23,037.7 & 5,148.0 & 16,200.0 & 20.6  & 16,256.5 & 14,522.3 & 10,469.0 \\

DQN+TER  &  2,208.0  & 4,721.3 & 24,332.4  & 5,032.4   & 16,632.0 & 21.0 & 17,281.3 & 19,232.5 & 10,834.0 \\
DQN+LaBER & 2,165.5  & 4,923.5 & 24,251.9  & 5,218.2   & 16,321.0 & 21.0 & 17,744.6 & 21,368.4  & 12,832.0 \\
DQN+Relo &  2,272.2  & 4,892.7 & 25,232.6  & 5,209.8   & \textbf{16,820.1} & 21.0 & 19,013.2 & 22,312.7  & 14,123.0 & \\
\midrule
DQN+EDER  w/o R.S. &  2,138.0  &5,145.4 & 24,214.2 & 5,121.0 & 16,054.2 & 21.0 & 18,421.0 & 21,833.1 & 13,233.1 & \\
DQN+EDER  w/o C.D. &  2,332.7  & 5,483.1 & 25,970.3 & 5,240.1 & 16,192.1 & 21.0 & 19,192.5 & 23,382.0 & 14,523.7 & \\
DQN+EDER & \textbf{2,340.0}    & \textbf{5553.0} & \textbf{26,246.0} & \textbf{5,275.0}  & 16,644.0 & \textbf{21.0} & \textbf{19,545.0} & \textbf{24,425.0} & \textbf{14,920.0}  &\\
\bottomrule 

\end{tabular}
\vspace{-5pt}
\caption{Comparison of Atari Game Scores. Best results are \textbf{bold}.}
\vspace{-10pt}
\label{tab:atari_scores}
\end{table*}
\vspace{-5pt}

\subsection{High-dimensional Environment}
\begin{wrapfigure}{r}{3.9cm}
    \includegraphics[scale=0.10]{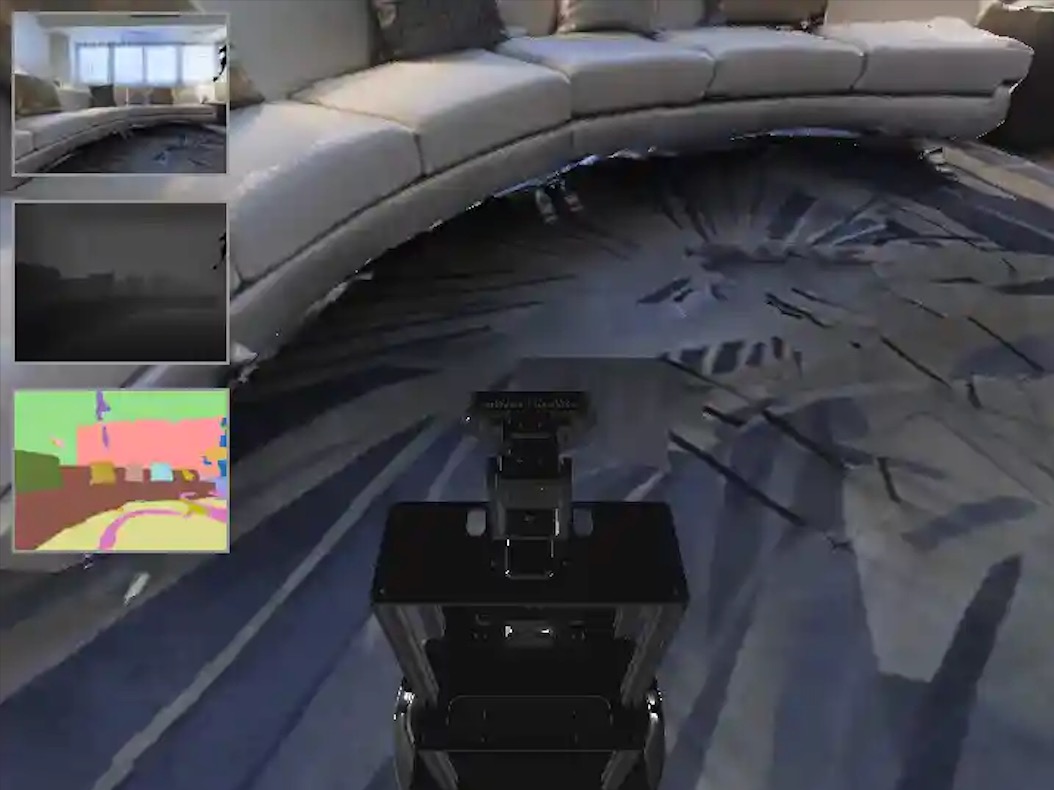}
    \caption{Habitat scene.}
    \label{fig:habitat}
    \vspace{-5pt}
\end{wrapfigure}
We utilize the AI Habitat platform to evaluate EDER’s scalability and effectiveness in vision-based navigation tasks. Specifically, the agent is randomly initialized in the environment and relies solely on its sensory inputs for navigation. With no prior knowledge of the environment map, the agent must autonomously explore the scene and locate the target object. The success metric is defined as whether the agent successfully reaches the target object.
These tasks are conducted in photorealistic 3D environments, where the high-dimensional observation space poses significant challenges for efficient exploration. We evaluate EDER in three environments from the Habitat-Matterport 3D Research Dataset (HM3D)~\cite{ramakrishnan2021habitatmatterport3ddatasethm3d}, representing complex, real-world indoor spaces. Specifically, we choose a residential setting (e.g., living rooms and bedrooms), an office environment (e.g., workspaces and corridors), and a commercial space (e.g., shopping centers), each featuring open areas and diverse visual elements. Target- or topology-based approaches (e.g., HER, TER) incorporate structured global exploration strategies in complex environments to improve efficiency, while loss- or priority-based methods (e.g., PER, LaBER, ReLo) focus on refining transition sampling based on loss or priority rules. However, whether by replacing goals or leveraging topological structures for exploration or by adjusting sampling mechanisms, these methods often neglect the importance of diverse trajectories and comprehensive exploration.
In contrast, EDER emphasizes leveraging diverse trajectories to enhance exploration efficiency. As shown in Table~\ref{tab:habitat_tasks}, EDER consistently achieves higher success rates across all experimental settings in high-dimensional visual tasks, demonstrating its effectiveness and scalability.

\subsection{Atari Games}

The second set of experiments evaluates EDER in discrete-action environments using the Atari benchmark, renowned for its challenging exploration tasks. For instance, in Alien, the agent navigates a maze, earns points by collecting bright spots, and loses a life upon contact with monsters. We test EDER+DQN across various Atari games, comparing it against standard DQN, DQN+PER, and other replay variants. The selected games, including Kangaroo and Jamesbond, are particularly demanding in terms of exploration. As shown in Table~\ref{tab:atari_scores}, EDER achieve the best performance in 17/18 games. Other methods generally fall into two categories: those that emphasize structured global exploration and those that prioritize samples based on local rewards or loss values. While effective in certain scenarios, these methods often overlook transitions that do not immediately yield high TD errors or rewards. In contrast, EDER explicitly incorporates diverse trajectories into its replay mechanism, promoting enhanced exploration and ultimately improving overall performance.

\begin{figure}[!htb]
\centering
\vspace{-5pt}
\subfloat[FetchPickAndPlace]{\includegraphics[width=0.23\textwidth]{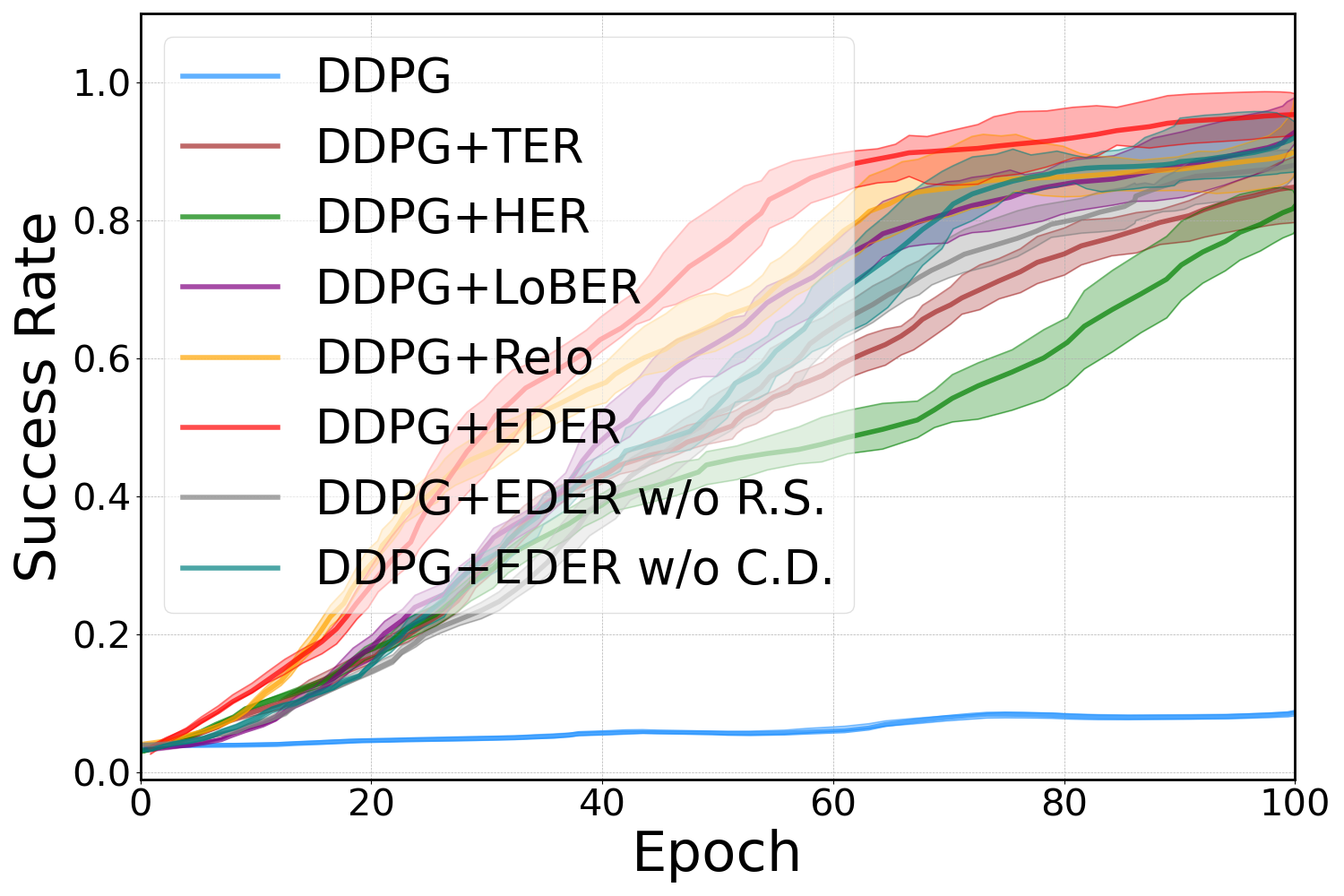}\label{fig:fetch_pick_place}}
\hfill
\subfloat[FetchPush]{\includegraphics[width=0.23\textwidth]{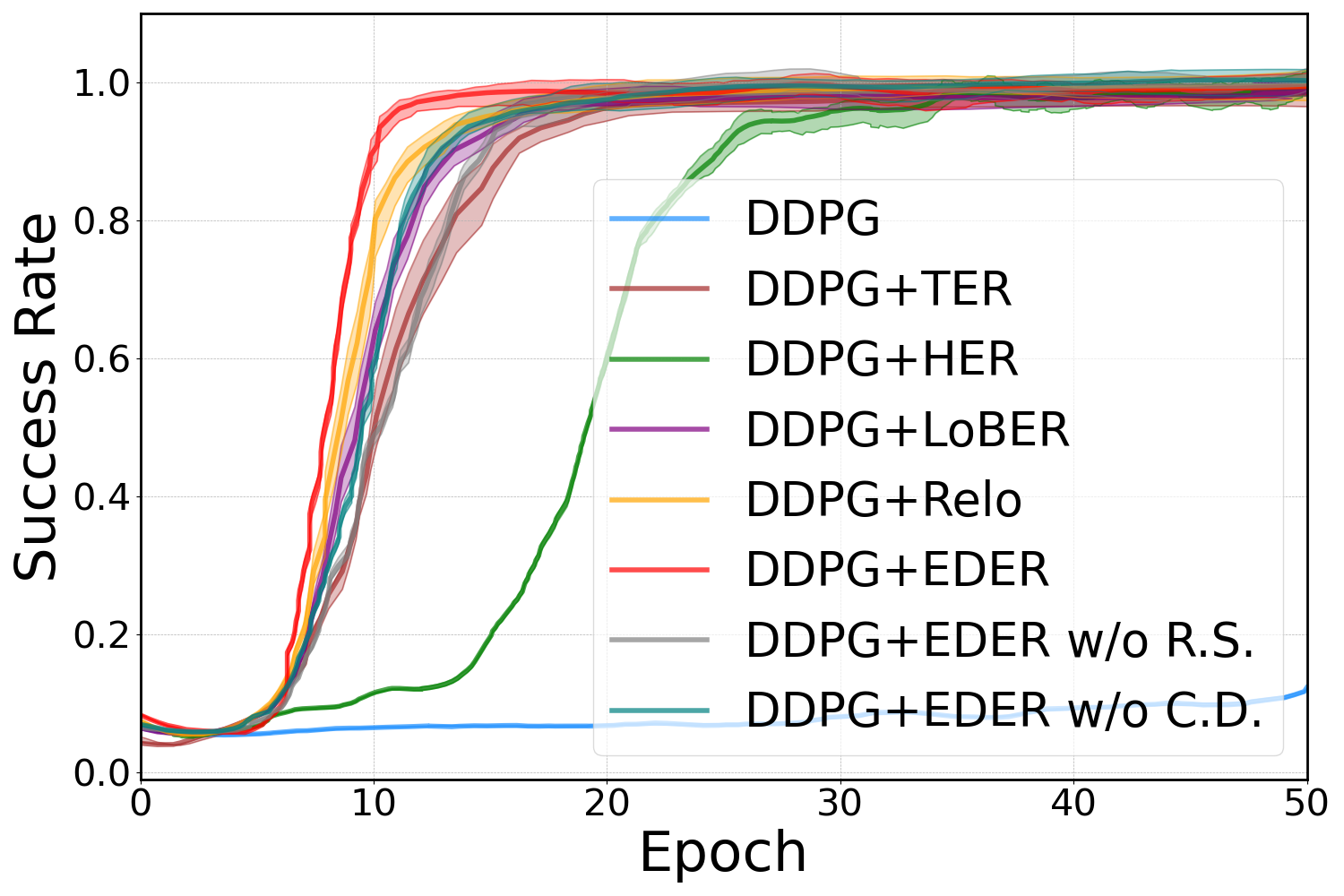}\label{fig:fetch_push}}
\\
\vspace{-5pt}
\subfloat[HandBlockRotate]{\includegraphics[width=0.23\textwidth]{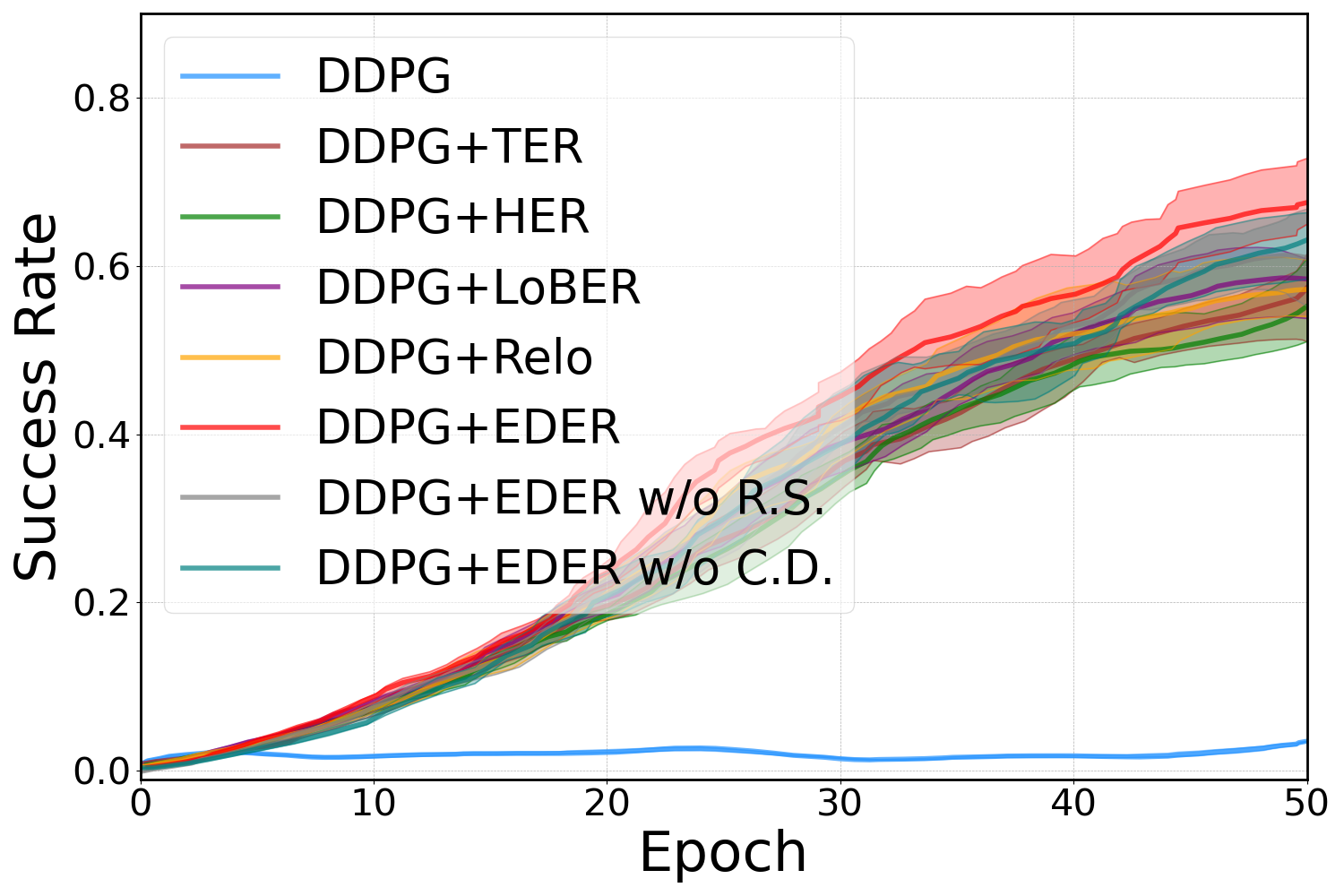}\label{fig:hand_block_rotate}}
\hfill
\subfloat[HandPenRotate]{\includegraphics[width=0.23\textwidth]{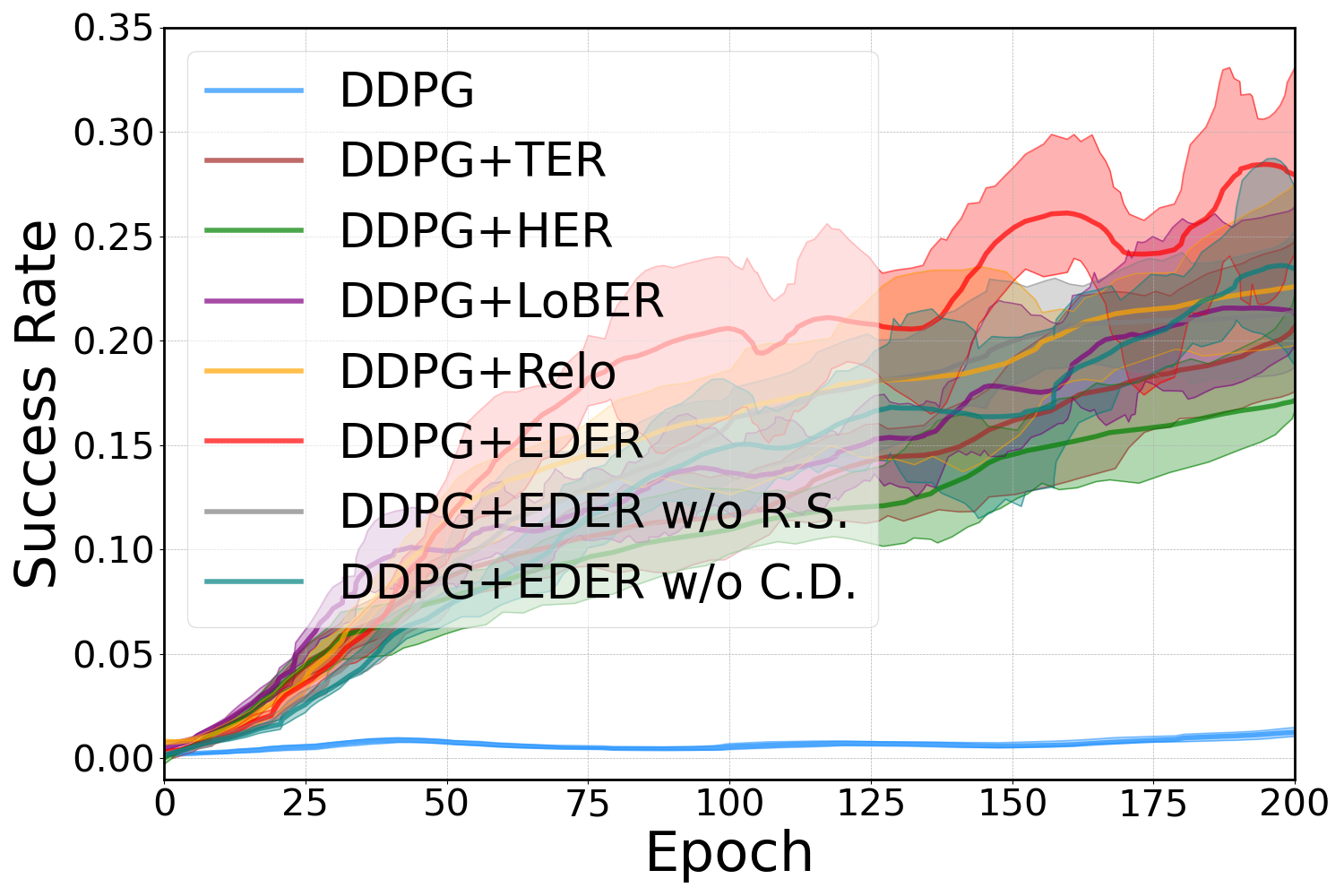}\label{fig:hand_pen_rotate}}
\vspace{-7pt} 
\caption{Success rates between EDER and other baselines}
\vspace{-10pt}
\label{fig:tasks}
\end{figure}
\subsection{Mujoco Tasks}

We also evaluate EDER in MuJoCo environments, focusing on continuous control tasks with sparse rewards. These tasks are particularly challenging due to their high-dimensional state and action spaces. We select four representative tasks from the FetchEnv, which involves a robotic arm with 7 degrees of freedom, and HandEnv, featuring the 24-degree-of-freedom Shadow Dexterous Hand: FetchPickAndPlace, FetchPush, HandBlockRotate, and HandPenRotate.
As shown in Figure~\ref{fig:tasks}, EDER significantly outperforms traditional DDPG and its variants in both learning speed and success rates. Notably, EDER achieves exceptional performance in the Shadow Dexterous Hand task, demonstrating its ability to navigate complex, high-dimensional spaces. This superior performance is attributed to EDER's capacity to enhance exploration through diverse trajectories, resulting in more efficient learning.

\subsection{Ablation Studies}
In our ablation studies, we evaluate the impact of Rejection Sampling (R.S.) and Cholesky Decomposition (C.D.) on the training efficiency of EDER. As shown in Table~\ref{tab:training_time_all}, removing R.S. reduces exploration efficiency, slows convergence, and lowers success rates. In contrast, excluding C.D. leads to increased instability and longer training times, ultimately diminishing learning efficiency.
Moreover, when comparing training times with HER-based baselines, EDER maintains competitive performance, especially in more complex tasks like PickAndPlace, demonstrating favorable trade-offs between computational cost and performance.
Additional experiments varying $m$ (the number of diverse trajectories) and $b$ (the trajectory length) are provided in Appendix~\ref{sec:additional_exp}. Our results indicate that increasing $m$ promotes exploration by sampling more diverse trajectories, helping to avoid local optima. However, excessive $m$ increases computational burden without proportionate performance gains, leading to slower convergence. Likewise, longer segments ($b$) provide richer context and improve exploration in temporally dependent environments, but overly long trajectories may introduce instability and reduce overall training efficiency.

\begin{table}[!htb]
\centering
\small
\setlength{\tabcolsep}{6pt}
\begin{tabular}{llll}
\toprule
\multicolumn{4}{c}{\textbf{PickAndPlace Task (Training Time in Minutes)}} \\
\midrule
\textbf{Method} & \textbf{Time} & \textbf{Method} & \textbf{Time} \\
\midrule
DDPG + HER (Buffer)  & 80.7  & DDPG + LaBER           & 93.6  \\
DDPG + PER (sum-tree) & 63.4  & DDPG + Relo            & 107.1 \\
DDPG + TER           & 91.9  & DDPG + EDER            & 103.1 \\
\midrule
\midrule
\multicolumn{4}{c}{\textbf{Push Task: DDPG + EDER Variants (in Minutes)}} \\
\midrule
\textbf{Method} & \textbf{Time} & \textbf{Method} & \textbf{Time} \\
\midrule
EDER (b=10)             & 124.3  & EDER (b=T)              & 156.0  \\
w/o R.S. (b=10)         & 173.2  & w/o R.S. (b=T)          & 182.7  \\
w/o C.D. (b=10)         & 129.2  & w/o C.D. (b=T)          & 171.3  \\
\bottomrule
\end{tabular}
\vspace{-5pt}
\caption{
Training times (in minutes) for baseline methods and EDER variants on PickAndPlace and Push tasks. R.S.: Rejection Sampling; C.D.: Cholesky Decomposition.
}
\label{tab:training_time_all}
\vspace{-5pt}
\end{table}

%% file: content/2_related_work.tex
\section{Related Work}
The concept of Experience Replay (ER) was first introduced by~\cite{lin1992self}, where past experiences are stored in a buffer and replayed during training to break the correlation between sequential data, which helps mitigate the non-stationarity in RL. \cite{mnih2013playingatarideepreinforcement} later incorporated ER into the Deep Q-Network (DQN), where the use of randomly sampled batches from the replay buffer was crucial in stabilizing the learning process and led to significant advancements in the performance of RL algorithms.

Prioritized Experience Replay (PER)~\cite{schaul2015prioritized} enhances learning by focusing on high TD-error samples and prioritizing informative experiences. Various extensions to PER have been proposed, such as the actor-critic-based PER~\cite{saglam2022actor}, which dynamically adjusts sampling priorities to balance exploration and exploitation; Attentive PER~\cite{sun2020attentive} uses attention mechanisms to replay experiences relevant to the current learning phase, enhancing training efficiency. Additionally, recent studies have introduced new priority criteria to enhance PER's effectiveness. Relo~\cite{sujit2023prioritizing} define the learnability of transitions as a priority criterion, prioritizing samples that consistently reduce training loss.TER~\cite{hong2022topological} builds a trajectory graph and prioritizes updates breadth-first from terminal states; LaBER~\cite{pmlr-v162-lahire22a} enhances efficiency by leveraging large batch sampling with focused updates. FSER~\cite{yu2024mixed} combines frequency and similarity indices to prioritize rare and policy-aligned experiences. ~\cite{wei2021deep} integrates transition revisit frequency with TD error for more effective replay buffer prioritization.

Hindsight Experience Replay (HER)~\cite{andrychowicz2017hindsight}, offers a novel approach to handling sparse rewards by retrospectively altering the goals of unsuccessful episodes, thereby converting failures into valuable learning experiences. HER has been integrated with techniques such as curriculum learning~\cite{fang2019curriculum} and multi-goal learning~\cite{plappert2018multigoal} to enhance the generalization and adaptability of RL agents. Additionally, Contact Energy Based Prioritization (CEBP)~\cite{sayar2024contact} prioritizes replay samples based on contact-rich interactions, selecting the most informative experiences. Distributed ER architectures like Ape-X~\cite{horgan2018distributed} and IMPALA~\cite{espeholt2018impala} have scaled experience replay across multiple actors, significantly accelerating training while maintaining efficiency. Relay Hindsight Experience Replay (RHER)~\cite{luo2023relay} decomposes tasks and employs a multi-goal network for self-guided exploration. Other methods also reuse past experience to learn policies by weighting sampled actions according to their estimated advantages~\cite{peng2019advantage}. These advancements enhance the robustness and scalability of experience replay methods, enabling more efficient and effective learning across a wide range of reinforcement learning tasks.

%% file: content/6_conclusion.tex
\section{Conclusion}

In this work, we present the Efficient Diversity-based Experience Replay (EDER) framework, which prioritizes sample diversity to significantly enhance the efficiency of experience replay (ER), particularly in high-dimensional state spaces and environments with sparse rewards. To address computational bottlenecks in large state spaces, we integrate Cholesky decomposition and rejection sampling, enabling the selection of more diverse and representative samples while optimizing the ER mechanism.
Extensive experiments on MuJoCo, Atari games, and Habitat demonstrate the superiority of EDER compared to existing approaches. EDER not only substantially improves learning efficiency but also delivers superior performance in high-dimensional and realistic environments. These results validate the effectiveness and adaptability of EDER across a variety of complex settings.

%% file: content/7_appendix.tex
\appendix

\section{Proofs}
\label{sec:proofs}

\begin{theorem}[Correlation between Determinant and Diversity]
Let $M \in \mathbb{R}^{d \times b}$ be a matrix whose columns are the $\ell_2$-normalized state vectors $\hat{s}$ in trajectory $\tau_j$. The determinant $\det(L_{\tau_j})$ of the kernel matrix $L_{\tau_j} = M^T M$ reaches its maximum value when the state vectors are mutually orthogonal, indicating the highest diversity of the trajectory.
\end{theorem}

\begin{proof}
\textbf{(Geometric Interpretation)}  
The determinant $\det(L_{\tau_j})$ represents the squared volume of the parallelepiped spanned by the columns of $M$. Specifically,
$$
\mathrm{Vol}_b(M) = \sqrt{\det(M^\top M)} = \sqrt{\det(L_{\tau_j})},
$$
so 
$$
\det(L_{\tau_j}) = \bigl[\mathrm{Vol}_b(M)\bigr]^2.
$$
Maximizing $\det(L_{\tau_j})$ corresponds to maximizing this volume, which is achieved when the columns of $M$ are orthogonal.

\medskip
\noindent
\textbf{(Algebraic Perspective)}  
If the columns of $M$ are orthonormal, then
$$
M^\top M = I_b \quad \Rightarrow \quad \det(L_{\tau_j}) = \det(I_b) = 1.
$$
Here, $I_b$ is the $b \times b$ identity matrix, and its determinant is 1.

Now, consider what happens when the columns of $M$ are not orthonormal. Suppose there exists at least one pair of columns that are not orthogonal. Let $\sigma_1, \sigma_2, \ldots, \sigma_b$ be the singular values of $M$. The determinant of $L_{\tau_j}$ can be expressed in terms of these singular values:
$$
\det(L_{\tau_j}) = \prod_{i=1}^b \sigma_i^2.
$$
Since the columns are $\ell_2$-normalized, the singular values satisfy $0 \leq \sigma_i \leq 1$. If all columns are orthonormal, then $\sigma_i = 1$ for all $i$, and hence $\det(L_{\tau_j}) = 1$. However, if any $\sigma_i < 1$, the determinant becomes
$$
\det(L_{\tau_j}) = \prod_{i=1}^b \sigma_i^2 < 1.
$$
This shows that any deviation from orthonormality reduces the determinant.

\medskip
\noindent

\textbf{Conclusion:}  
The determinant $\det(L_{\tau_j})$ is maximized when the columns of $M$ are orthonormal, achieving the maximum possible value of 1, which corresponds to the highest diversity among the state vectors.
$$
\boxed{
  \max\,\det(L_{\tau_j}) = 1 \quad \text{when columns of } M \text{ are orthonormal.}
}
$$
\end{proof}

\begin{theorem}[Time Complexity of EDER]
\label{theorem2}
The time complexity of the EDER algorithm is $ O(N b d + N b^3 + N \log m + m) $ without employing Cholesky decomposition and rejection sampling, and it is reduced to $ O(N b d + N b^2 + N \log m + m) $ after integrating these optimizations. Here, $ N $ denotes the number of state transitions, $ b $ the segment length, $ d $ the dimensionality of the state vectors, and $ m $ the number of top trajectories selected.
\end{theorem}

\begin{proof}
\textbf{(Trajectory Segmentation)}  
Dividing $N$ state transitions into segments of length $b$ requires $O(N)$ time.

\medskip
\noindent
\textbf{(Kernel Matrix Construction)}  
For each of the $\tfrac{N}{b}$ segments, constructing $L_{\tau_j} = M^\top M$ involves $O\bigl(b^2 d\bigr)$ operations. The total time is therefore 
$$
O\Bigl(\tfrac{N}{b} \cdot b^2 d\Bigr) \;=\; O\bigl(N\,b\,d\bigr).
$$

\medskip
\noindent
\textbf{(Determinant Calculation)}  
\begin{itemize}
\item \emph{Without Cholesky.} Computing $\det(L_{\tau_j})$ via standard methods (e.g.\ LU decomposition) takes $O\bigl(b^3\bigr)$ per segment. The total cost is 
$$
O\Bigl(\tfrac{N}{b} \cdot b^3\Bigr) \;=\; O\bigl(N\,b^3\bigr).
$$

\item \emph{With Cholesky.} Performing Cholesky decomposition,
$$
L_{\tau_j} \;=\; L_C\,L_C^\top,
$$
where $L_C$ is a lower triangular matrix. The determinant becomes
$$
\det(L_{\tau_j}) \;=\; \prod_{i=1}^b l_{ii}^2,
$$
reducing the computation to $O\bigl(b^2\bigr)$ per segment, giving
$$
O\Bigl(\tfrac{N}{b} \cdot b^2\Bigr) \;=\; O\bigl(N\,b^2\bigr).
$$
\end{itemize}

\medskip
\noindent
\textbf{(Sampling and Sorting)}  
Selecting the top $m$ trajectories (out of $N$) based on diversity scores can be done in $O(N \log m)$ time.

\medskip
\noindent
\textbf{(Model Update)}  
Updating the model with $m$ selected trajectories takes $O(m)$ time.

\medskip
\noindent
\textbf{Conclusion:} 
Summing all components yields 
$$
O\bigl(N\,b\,d + N\,b^3 + N\,\log m + m\bigr)
\quad \text{vs.} \quad
O\bigl(N\,b\,d + N\,b^2 + N\,\log m + m\bigr).
$$
Hence, employing Cholesky decomposition reduces the determinant computation from $O\bigl(N\,b^3\bigr)$ to $O\bigl(N\,b^2\bigr)$.

\medskip
\noindent
\textbf{Cholesky Decomposition Formula.} 
For any window length $b$, let $M = [\hat{s}_1^{ac}, \hat{s}_2^{ac}, \ldots, \hat{s}_b^{ac}]$. The kernel matrix 
$$
L_{\tau_j} = M^\top M \quad \text{has entries} \quad L_{ij} = \hat{s}_i^{ac}\cdot\hat{s}_j^{ac}.
$$
Cholesky yields
$$
L_{\tau_j} = L_C\,L_C^\top,
\quad
L_C = \begin{bmatrix}
l_{11} & 0 & \cdots & 0 \\
l_{21} & l_{22} & \cdots & 0 \\
\vdots & \vdots & \ddots & \vdots \\
l_{b1} & l_{b2} & \cdots & l_{bb}
\end{bmatrix},
$$
where
$$
l_{ii} \;=\; \sqrt{L_{ii} - \sum_{k=1}^{i-1} l_{ik}^2},
\quad
l_{ij} \;=\; \frac{L_{ij} - \sum_{k=1}^{j-1} l_{ik}\,l_{jk}}{\,l_{jj}\,}\quad (\text{for } j<i).
$$
Then
$$
\det(L_{\tau_j}) \;=\; \prod_{i=1}^b l_{ii}^2.
$$
Cholesky decomposition simplifies the determinant computation, improving efficiency and numerical stability.
\end{proof}

\section{Additional Experiments}
\label{sec:additional_exp}

In this section, we supplement the main text with additional experimental results related to the EDER algorithm, aiming to provide a more comprehensive demonstration of its performance across different environments and parameter settings. These experiments include: parameter analysis in the Mujoco robotic arm environment, extensive testing in a broader range of Atari games, sensitivity analysis of key parameters, and additional experiments in the Habitat simulation environment.

\subsection{Complete Results in Habitat Environment}

This section provides the complete experimental results of EDER in all tested environments to further validate its scalability and effectiveness in high-dimensional visual navigation tasks. Through a series of experiments conducted on the AI Habitat platform, we evaluated EDER's performance in realistic 3D environments, where the tasks involve high-dimensional observation spaces, making efficient exploration particularly challenging. The experiments were conducted in three different environments from the Habitat-Matterport 3D Research Dataset (HM3D), which simulate complex real-world indoor spaces. The environments selected include:

\begin{itemize}
    \item \textbf{Residential Environment}: Typical household spaces such as living rooms and bedrooms, used to assess EDER's navigation capabilities in common residential scenarios.
    \item \textbf{Office Environment}: Including workspaces, meeting rooms, and corridors, testing EDER's adaptability in office spaces.
    \item \textbf{Commercial Environment}: Simulating shops or shopping centers with open areas and various visual elements, further evaluating EDER's performance in diverse environments.
\end{itemize}

In each environment, we compared EDER with multiple baseline methods, particularly DDPG and its variants. Fig.~\ref{fig:m} shows the comparison of success rates across all environments, indicating that EDER consistently outperformed baseline methods, achieving higher success rates and demonstrating robustness in diverse scenarios. These results further confirm EDER's scalability and practicality in high-dimensional visual tasks, especially in complex and realistic 3D environments where it effectively navigates and explores, significantly enhancing the overall performance of reinforcement learning algorithms.

\begin{figure}[H]
\centering
\subfloat[DDPG]{\includegraphics[width=0.33\textwidth]{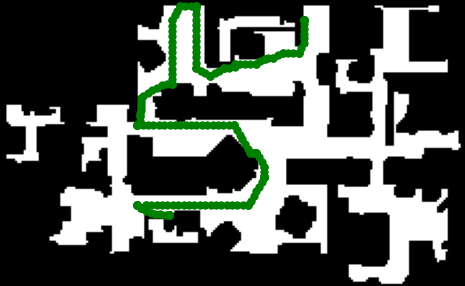}\label{fig:ddpg11}}
\hfill
\subfloat[DDPG+HER]{\includegraphics[width=0.33\textwidth]{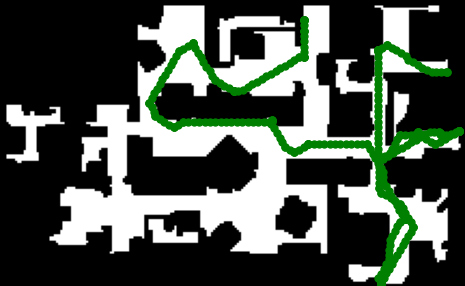}\label{fig:her11}}
\hfill
\subfloat[DDPG+TER]{\includegraphics[width=0.33\textwidth]{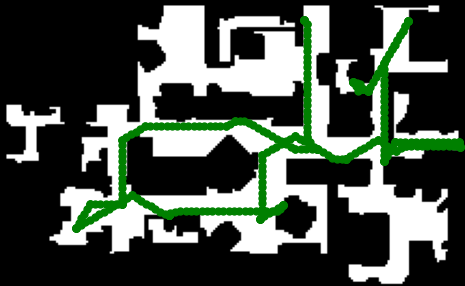}\label{fig:ter11}}
\hfill
\\
\subfloat[DDPG+LoBER]{\includegraphics[width=0.33\textwidth]{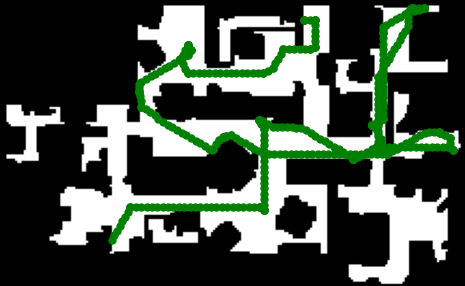}\label{fig:lober11}}
\hfill
\subfloat[DDPG+Relo]{\includegraphics[width=0.33\textwidth]{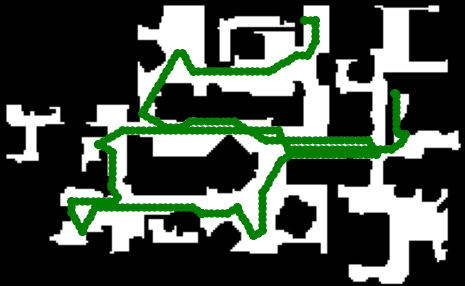}\label{fig:relo11}}
\hfill
\subfloat[DDPG+EDER]{\includegraphics[width=0.33\textwidth]{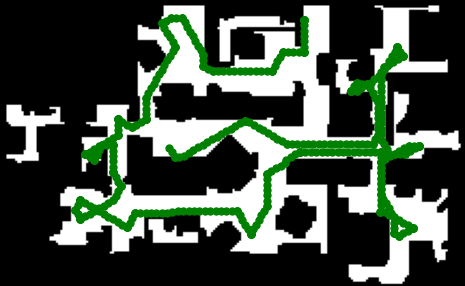}\label{fig:eder11}}
\caption{Trajectories of policies trained with different exploration algorithms on the Habitat environment}
\label{fig:aaaa}
\end{figure}

\subsection{Complete Results in Atari Games}

This section presents the complete experimental results of EDER across all tested Atari games. These experiments aim to evaluate EDER's performance in handling sparse rewards and complex exploration tasks. The experiments were conducted in the OpenAI Gym environment, focusing on simulated Atari game tasks, where each game is governed by specific rules, and the agent interacts with the environment through a limited set of discrete actions. Observations typically include pixel data and game scores, with the task goal usually being to optimize the game score or complete specific challenges. The reward system in most Atari games is designed based on game scores, providing feedback for the algorithm to optimize behavior. All agents were trained under the same configuration, using a minibatch size of 32 and evaluating the model every 5,000 training steps before reaching a maximum of 1 million training steps. The experiments were conducted using PyTorch-based DQN and its variants, including Random, standard DQN, DQN+PER, DQN+TER, DQN+LaBER, DQN+Relo, and DQN+EDER method, as well as EDER without rejection sampling (w/o R.S.) and without Cholesky decomposition (w/o C.D.). In the Atari game experiments, DQN was chosen due to its suitability for handling discrete action spaces. EDER outperformed traditional DQN and its variants in most Atari environments, particularly in complex exploration tasks like Montezuma's Revenge, where EDER effectively discovered sparse rewards and significantly improved scores. Overall, EDER demonstrated superior performance in these games, showcasing stronger learning capabilities and adaptability in environments with sparse rewards and complex structures.

\subsection{Complete Results in Mujoco Experiments}

In this series of experiments, we introduced the self-designed EDER+DDPG method and rigorously compared it with the traditional DDPG algorithm and its variants, including DDPG, DDPG+PER, DDPG+TER, DDPG+LaBER, DDPG+Relo, and DDPG+EDER method, as well as EDER without rejection sampling (w/o R.S.) and without Cholesky decomposition (w/o C.D.). The experiments were conducted in the Mujoco simulation environment, specifically focusing on the challenging Fetch robotic arm and Shadow Dexterous Hand tasks. Each method was evaluated using five different random seeds to ensure the robustness and reliability of the results. In the Fetch robotic arm tasks, which test both control precision and exploration efficiency, EDER+DDPG consistently outperformed other methods. The results showed higher success rates across all Fetch tasks. The EDER+DDPG method also demonstrated faster learning, attributed to its enhanced exploration capabilities, allowing the agent to discover rewarding states more efficiently in high-dimensional spaces. The use of Determinantal Point Processes (DPPs) in EDER improved exploration efficiency by promoting diversity in experience replay, leading to better generalization and policy learning.

In the Shadow Dexterous Hand tasks, which involve a robotic hand with 24 degrees of freedom and require complex manipulations, EDER+DDPG again outperformed the baseline methods. The method achieved significantly higher success rates in tasks like HandBlockRotate and HandPenRotate, which require intricate manipulation. The superior exploration strategy of EDER+DDPG, enabled by DPP-based sampling, allowed it to effectively navigate the large and complex state-action space, even in these high-dimensional tasks. Despite the complexity, EDER+DDPG maintained competitive training times, demonstrating that the additional computational overhead is justified by the substantial improvements in learning outcomes and task performance. Overall, the results from both FetchEnv and HandEnv validate the effectiveness of EDER+DDPG, highlighting its strong capabilities and broad applicability in complex continuous control tasks.

\subsection{Ablation Experiments}
\label{sec:proofs_ablation}
In this section, we conduct a series of ablation experiments to deeply explore the specific contributions of each key component of the EDER method to model performance. These experiments not only help us understand the role of each component in the overall method but also provide a basis for further algorithm optimization.

\stitle{Impact of Sampling Number $ m $}

In the EDER method, the sampling number $ m $ determines the number of samples extracted in each minibatch, which has a crucial impact on the algorithm's stability and convergence speed. To analyze the direct effect of this parameter on learning outcomes, we adjusted the value of $ m $ in a series of predetermined experiments and observed significant differences in model performance under different settings. In the experiments, the sampling number $ m $ was set to 50, 100, 200, 300 and 500 respectively, and EDER's performance was tested across multiple environments. The results show that larger $ m $ values generally improve learning stability and convergence speed, but also increase computational burden. Particularly in environments with sparse rewards, increasing the sampling number helps the model find effective strategies faster. Figure A.4.1 shows the success rate curves of EDER in various tasks with different $ m $ values, indicating that the model's performance is optimal at $ m = 200 $.

\begin{figure}[H]
\centering
\subfloat[FetchPickAndPlace]{\includegraphics[width=0.24\textwidth]{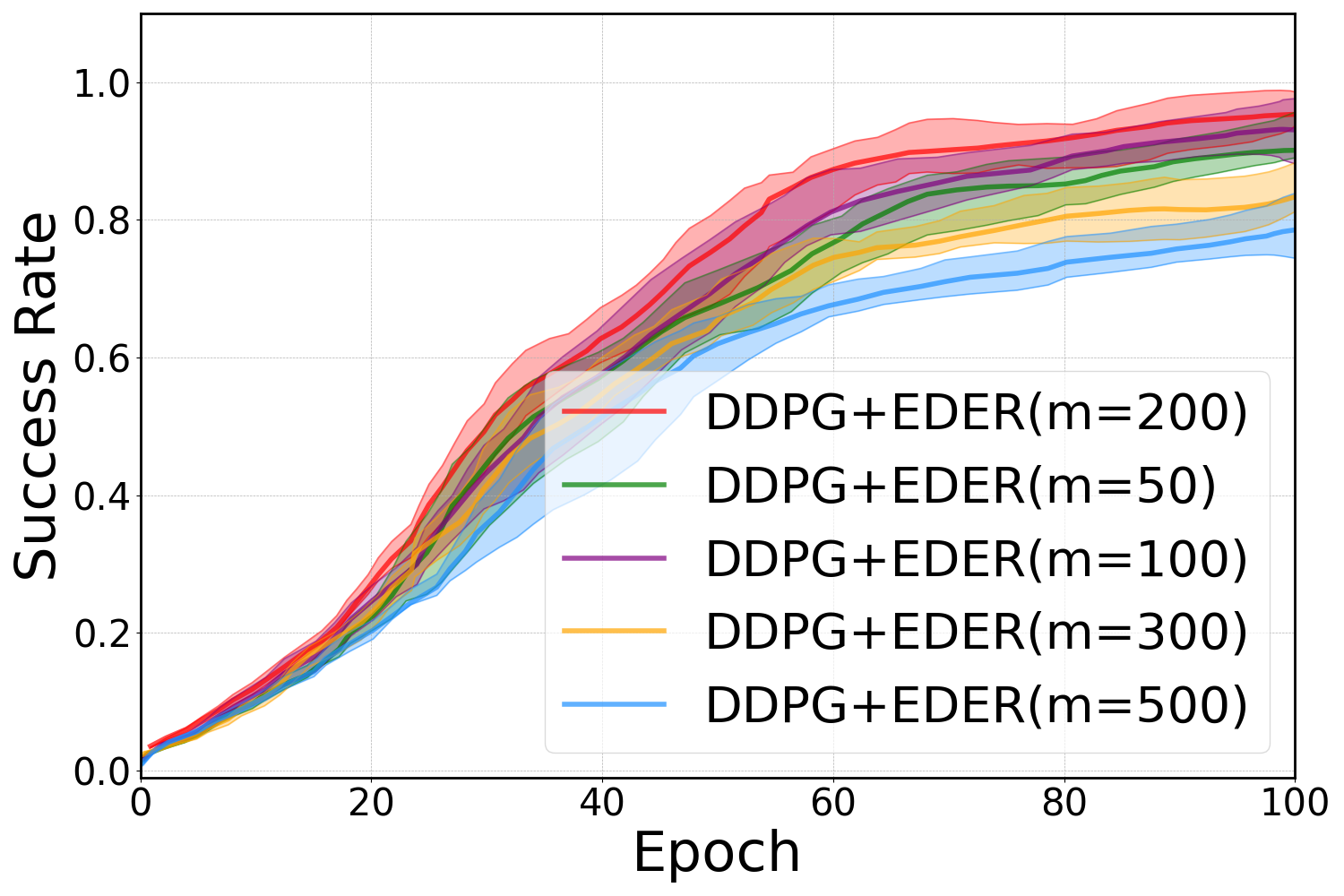}\label{fig:fetch_pick_place1}}
\hfill
\subfloat[FetchPush]{\includegraphics[width=0.24\textwidth]{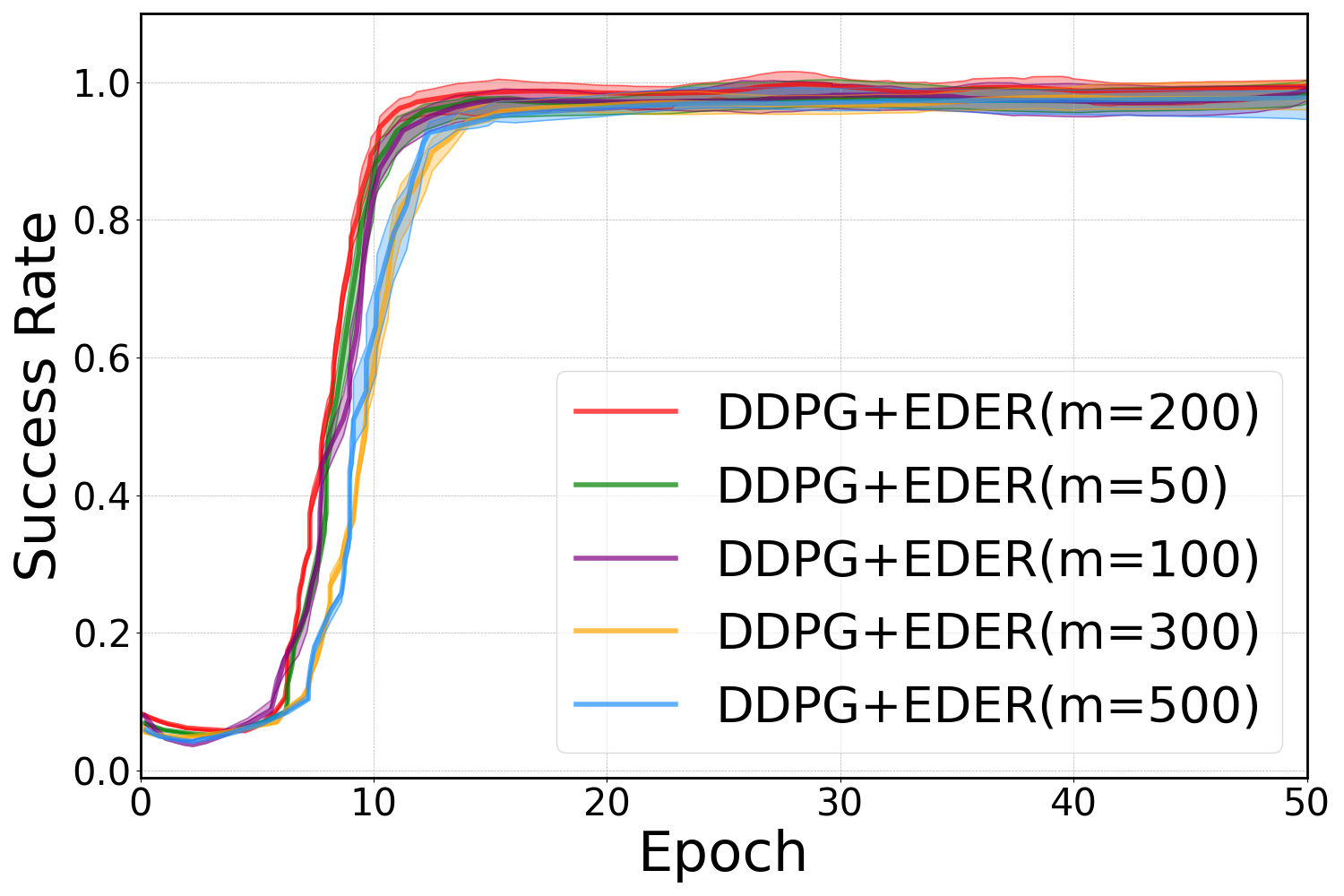}\label{fig:fetch_push1}}
\hfill
\subfloat[HandBlockRotate]{\includegraphics[width=0.24\textwidth]{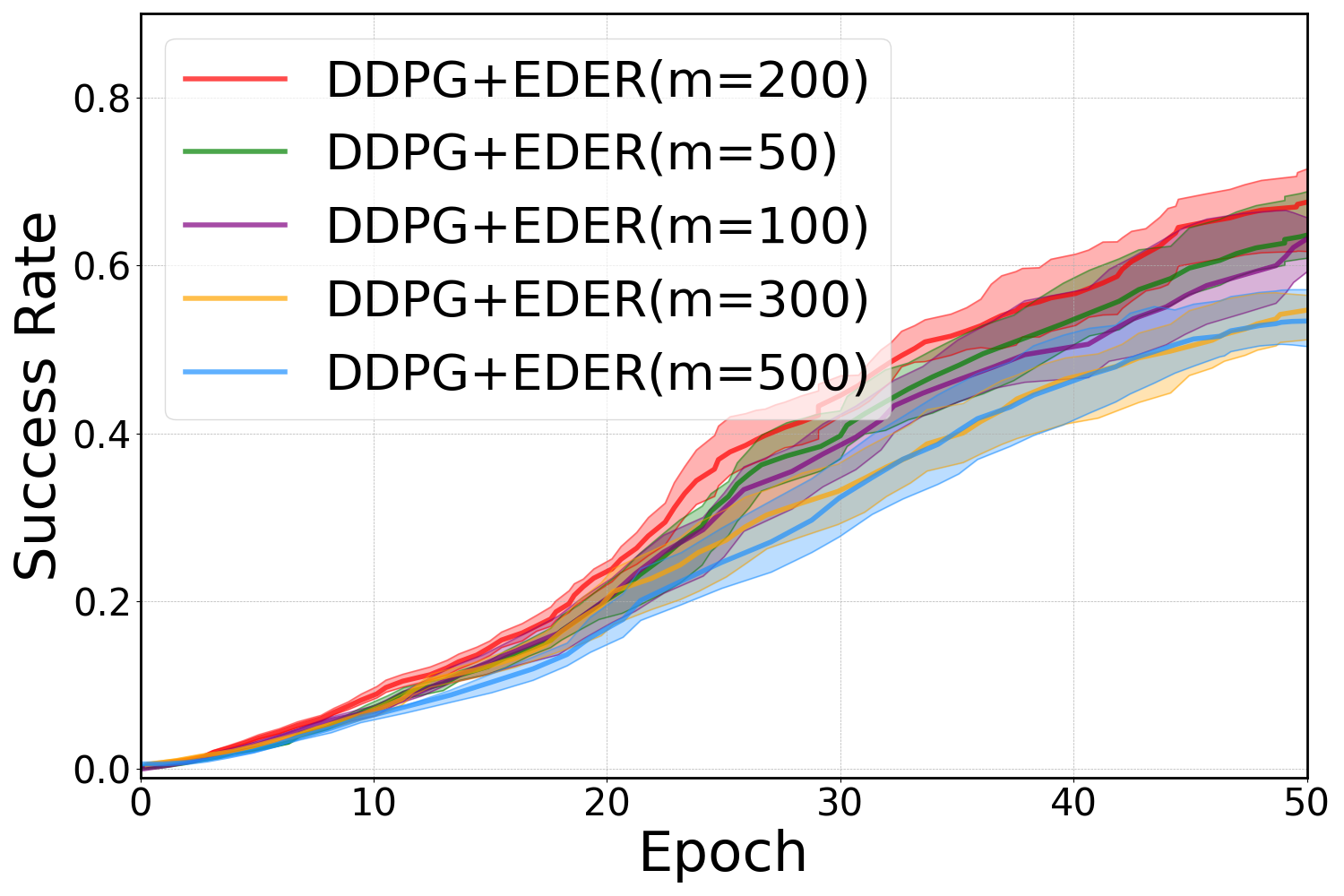}}\label{fig:hand_block_rotate1}
\hfill
\subfloat[HandPenRotate]{\includegraphics[width=0.24\textwidth]{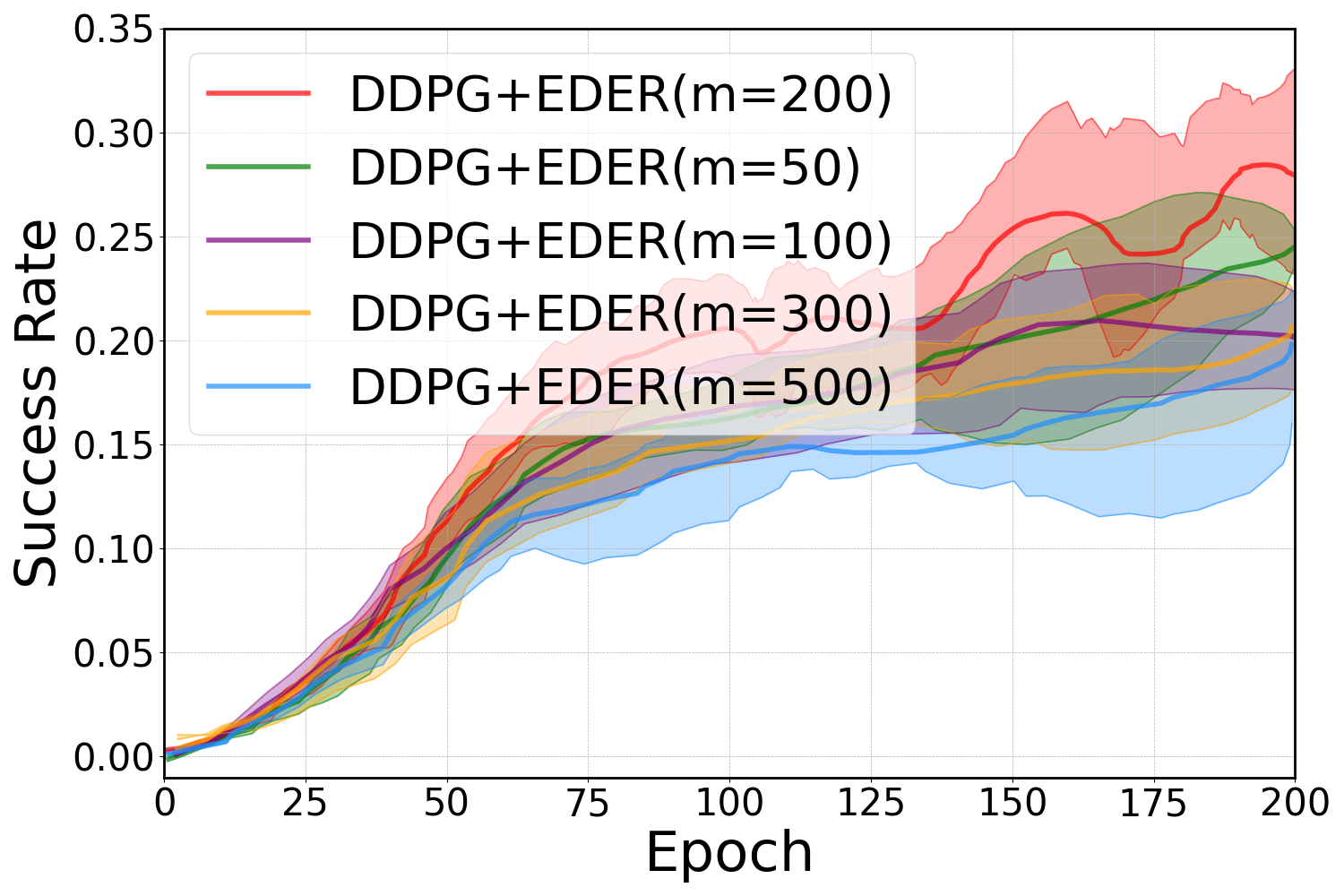}}\label{fig:hand_pen_rotate1}

\caption{Success Rate Comparison between EDER and Other Baselines in MuJoCo for number M}
\label{fig:m}
\end{figure}

\stitle{Impact of Trajectory Length $ b $}

Trajectory length $ b $ is a key parameter in reinforcement learning tasks as it determines how much information the agent can gather from the environment. To investigate the impact of this parameter on learning performance, we designed a series of experiments by varying the trajectory length $ b $ and observing its effect on the model's performance.In the experiments, the trajectory length $ b $ was set to 2, 5, 10, and 20, and EDER's performance was evaluated in different test environments. Longer trajectories typically provide more comprehensive information about the environment's dynamics, aiding strategy learning, but also increase computational complexity. The results show that a moderate trajectory length (e.g., $ b = 10 $) performs best in most cases, achieving a good balance between learning efficiency and computational burden. Figure ~\ref{fig:b} presents the learning curves of EDER with different trajectory lengths, further validating this conclusion.

\begin{figure}[H]
\centering
\subfloat[FetchPickAndPlace]{\includegraphics[width=0.24\textwidth]{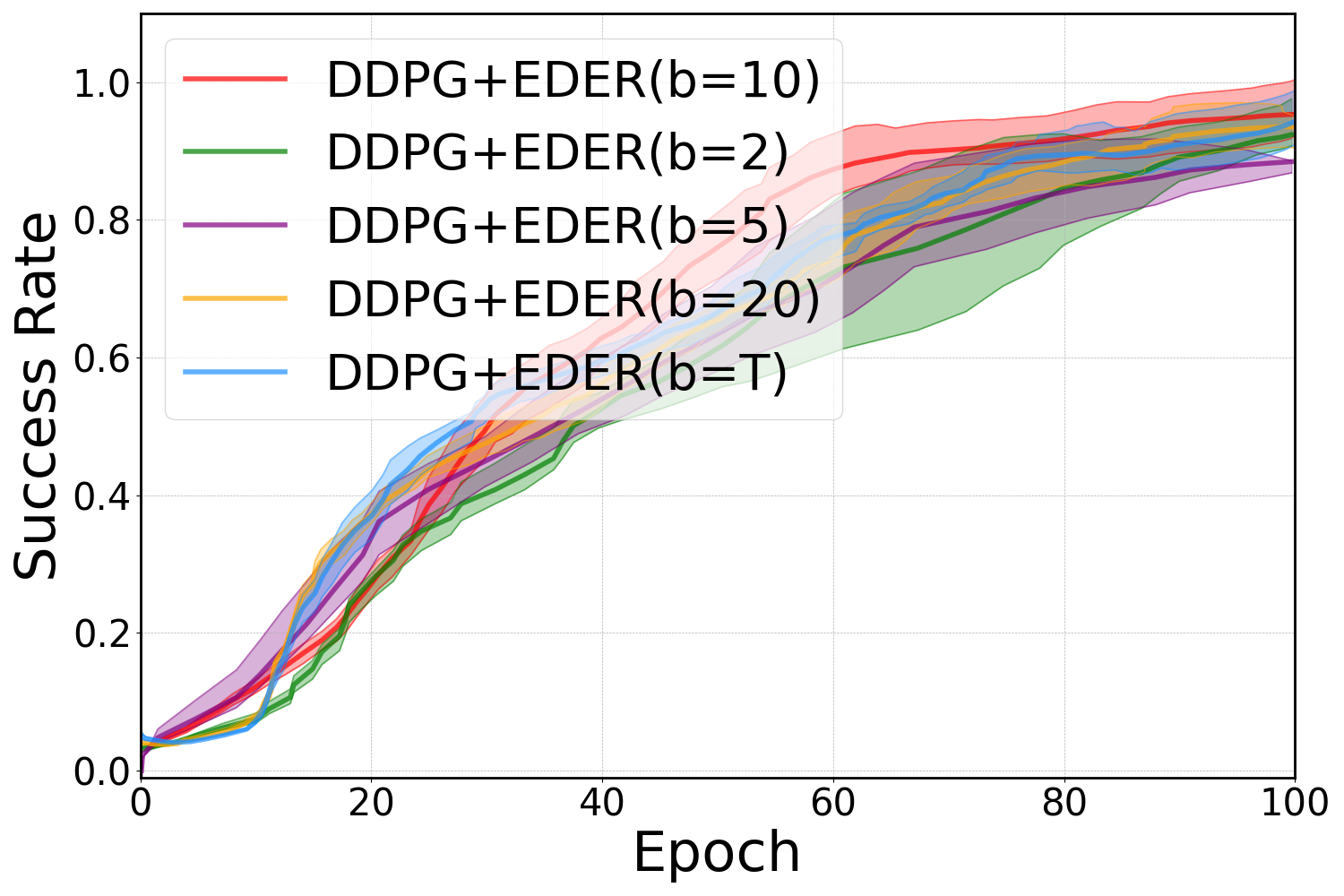}\label{fig:fetch_pick_place2}}
\hfill
\subfloat[FetchPush]{\includegraphics[width=0.24\textwidth]{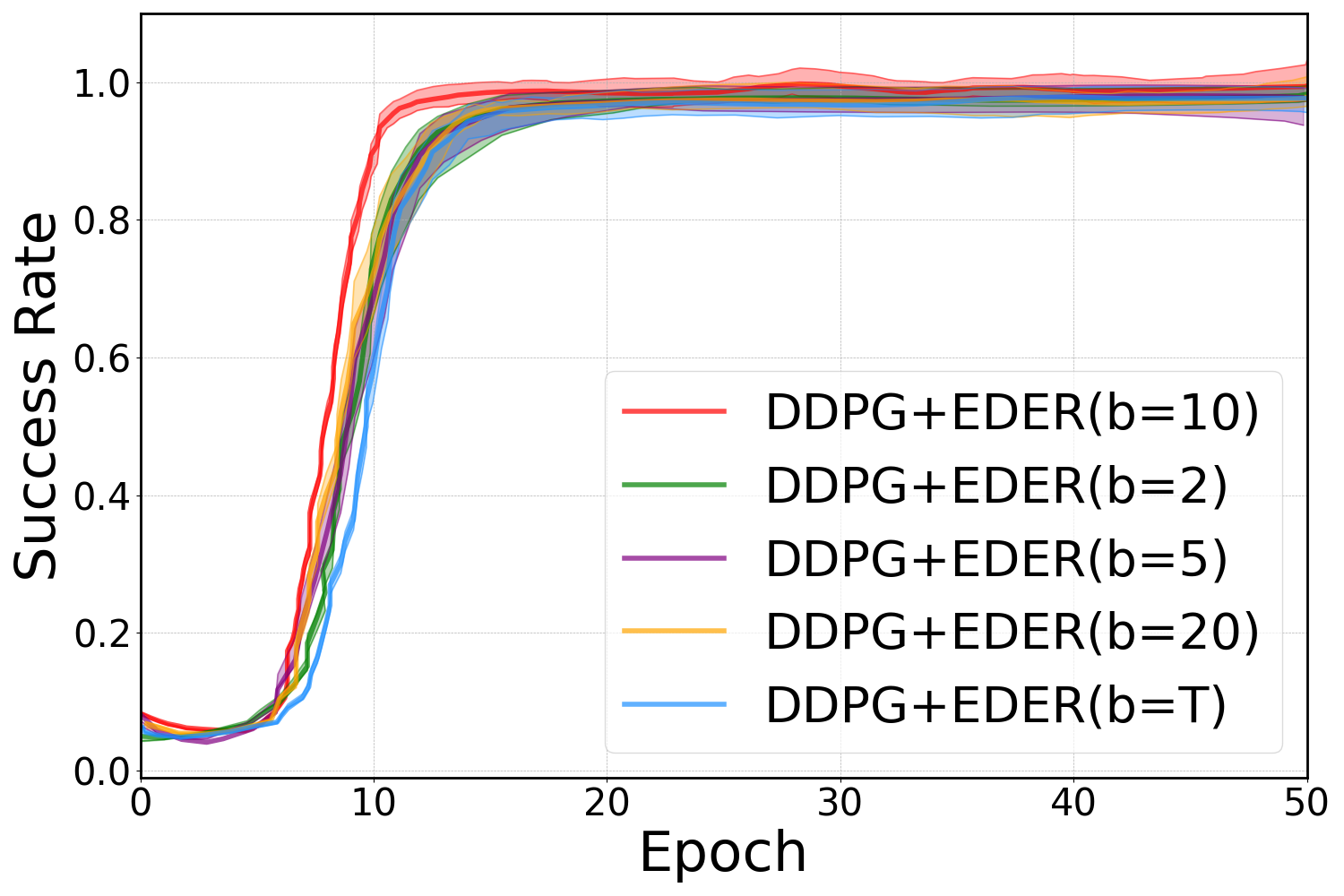}\label{fig:fetch_push2}}
\hfill
\subfloat[HandBlockRotate]{\includegraphics[width=0.24\textwidth]{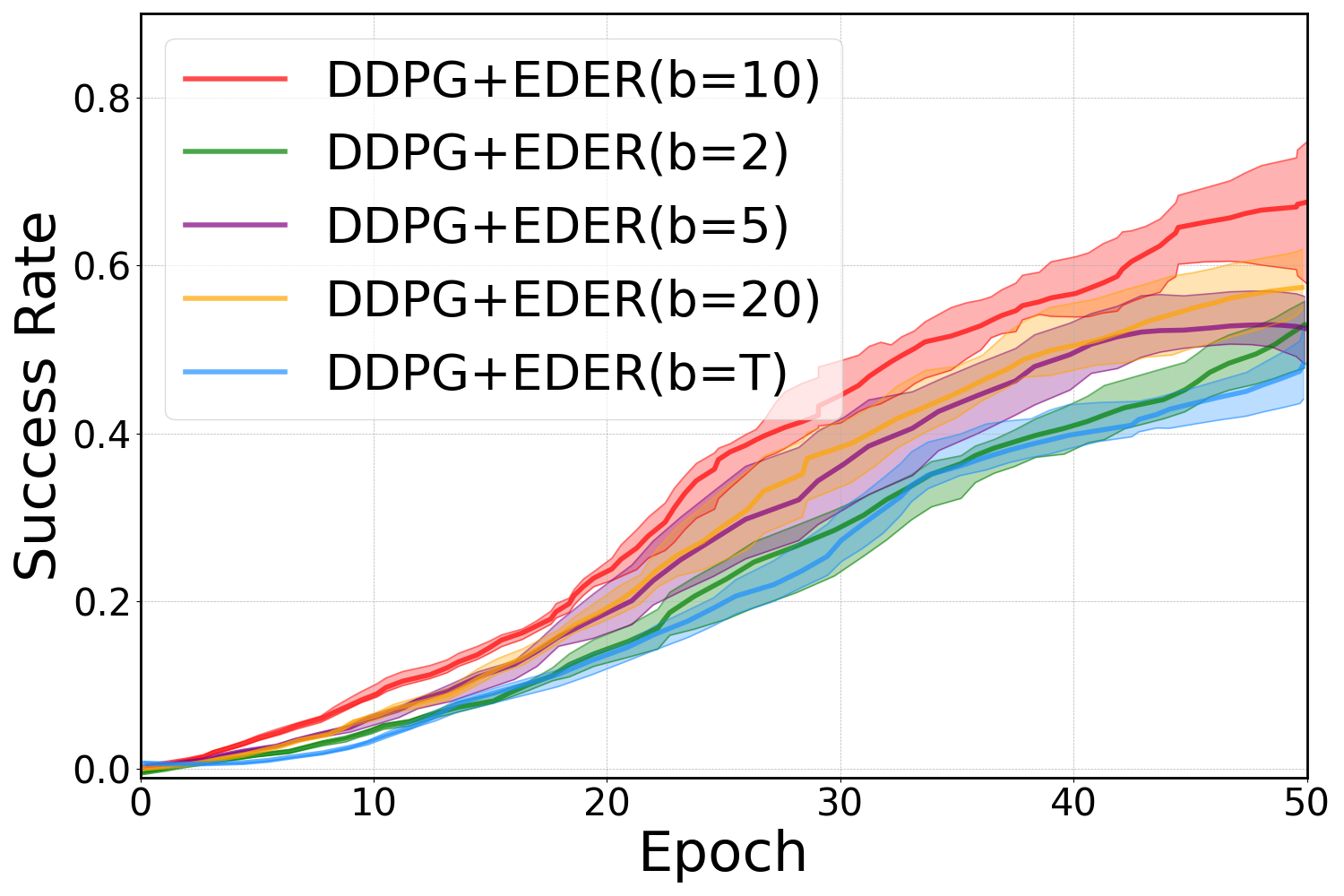}}\label{fig:hand_block_rotate2}
\hfill
\subfloat[HandPenRotate]{\includegraphics[width=0.24\textwidth]{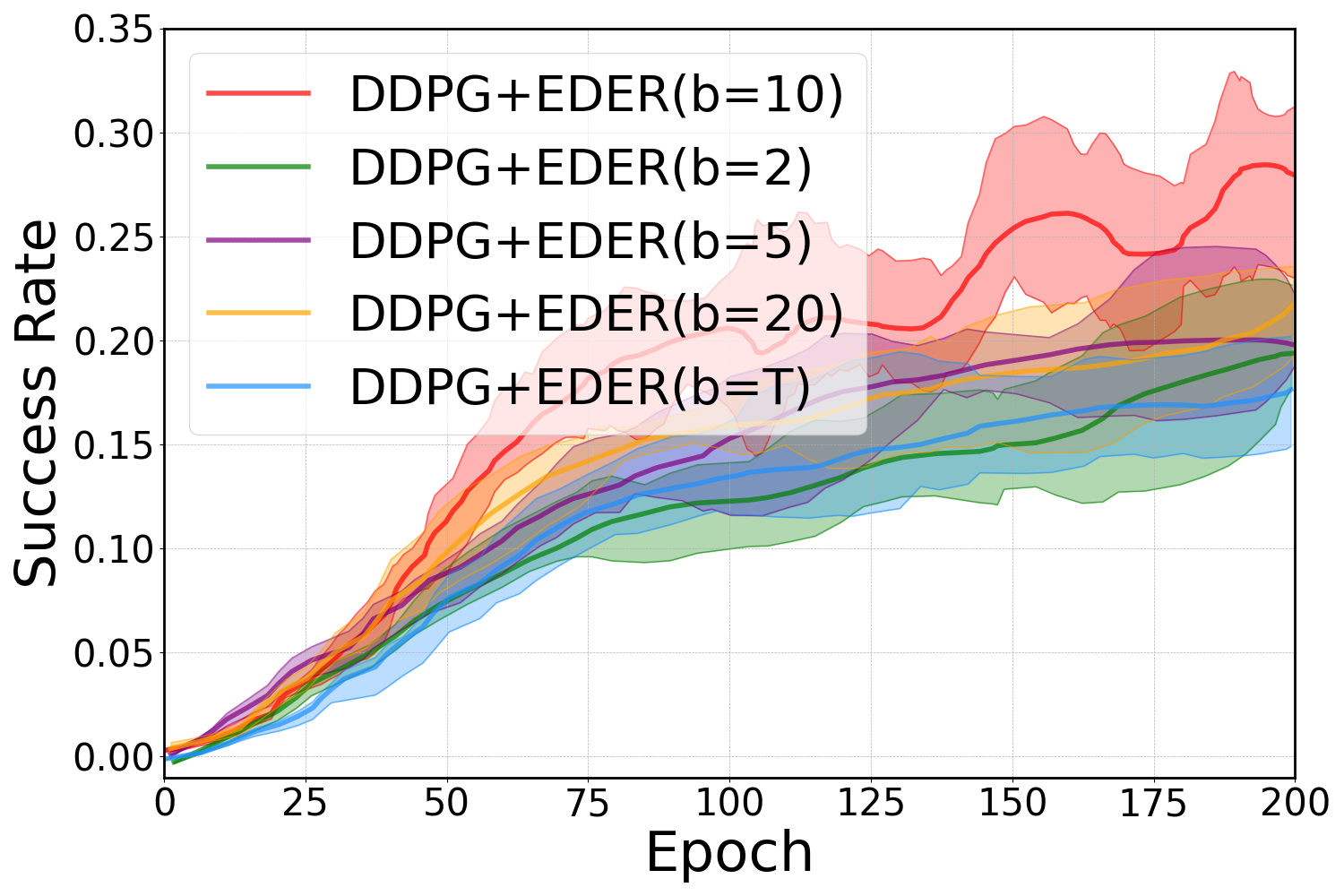}}\label{fig:hand_pen_rotate2}
\caption{Success Rate Comparison between EDER and Other Baselines in MuJoCo for length b}
\label{fig:b}
\end{figure}

\stitle{Impact of Rejection Sampling}

Rejection sampling is an important technique used in the EDER method to optimize the sampling process by prioritizing trajectory segments with higher diversity. To evaluate the impact of rejection sampling on algorithm performance, we conducted comparison experiments with and without rejection sampling.In the experiments, we kept other parameters constant and only varied the presence of rejection sampling. The results show that using rejection sampling significantly improved EDER's learning efficiency in complex environments, especially in tasks requiring deep exploration. Rejection sampling effectively selects representative samples, accelerating the strategy learning process. Table presents the learning curves in both scenarios, demonstrating the critical role of rejection sampling in enhancing EDER's performance.

\stitle{Impact of Cholesky Decomposition}

Cholesky decomposition is an optimization technique used to improve computational efficiency, particularly when handling high-dimensional state spaces. To assess the impact of Cholesky decomposition on algorithm performance and efficiency, we conducted additional ablation experiments, comparing the algorithm's runtime and learning outcomes with and without using Cholesky decomposition.In the experiments, we tested the cases with path length $ b = 10 $ and $ b = T $ (the entire trajectory length), with and without Cholesky decomposition. The results show that using Cholesky decomposition significantly improves computational efficiency, especially when dealing with larger trajectory scales. Cholesky decomposition reduces computational overhead and improves numerical stability. Table presents the training time and convergence curves in different scenarios, further confirming the importance of Cholesky decomposition in complex environments.Through the above ablation experiments, we thoroughly explored the impact of key components in the EDER method, including the sampling number $ m $, trajectory length $ b $, the presence of rejection sampling, and the role of Cholesky decomposition. These experiments not only help us understand the specific contributions of each parameter and technique to the algorithm's performance but also provide clear guidance on how to adjust these parameters to optimize learning outcomes in practical applications. This research is crucial for further enhancing EDER's performance in reinforcement learning tasks.

\begin{table}[H]
\centering
\setlength{\tabcolsep}{2pt}
\small
\begin{tabular}{lclc}
\toprule
\textbf{Configuration($ m = 200 $)} & \textbf{Time (hh:mm:ss)} & \textbf{Configuration($ b = 10 $)} & \textbf{Time (hh:mm:ss)} \\
\midrule

DDPG+EDER ($ b = 10 $) & 02:04:20  & DDPG+EDER ($ m = 300 $) & 02:41:04  \\
DDPG+EDER ($ b = 10 $, w/o R.S.) & 02:53:10 & DDPG+EDER ($ m = 300 $, w/o R.S.) & 03:28:32 \\
DDPG+EDER ($ b = 10 $, w/o C.D.) & 02:09:13 & DDPG+EDER ($ m = 300 $, w/o C.D.) & 03:07:39 \\
DDPG+EDER ($ b = 10 $, w/o R.S., w/o C.D.) & 03:25:21 & DDPG+EDER ($ m = 300 $, w/o R.S., w/o C.D.) & 03:51:06 \\
DDPG+EDER ($ b = T $) & 02:36:00 & DDPG+EDER ($ m = 500 $) & 03:45:53 \\
DDPG+EDER ($ b = T $, w/o R.S.) &03:02:41 & DDPG+EDER ($ m = 500 $, w/o R.S.) & 04:26:14 \\
DDPG+EDER ($ b = T $, w/o C.D.) & 02:51:20 & DDPG+EDER ($ m = 500 $, w/o C.D.) & 04:01:49 \\
DDPG+EDER ($ b = T $, w/o R.S., w/o C.D.) & 03:55:55 & DDPG+EDER ($ m = 500 $, w/o R.S., w/o C.D.) &  04:58:02\\
\bottomrule
\end{tabular}
\vspace{-5pt} 
\caption{Training times for EDER with and without Rejection Sampling and Cholesky Decomposition on Push task. The results highlight the impact of these optimizations on training efficiency.}
\vspace{-5pt} 
\label{tab:ablation_training_time}
\end{table}

\section{Experimental Implementation Details}
\label{sec:ex_imple_detail}

\subsection{Environment Setup}

In this section, we provide a detailed description of the experimental setups in the Atari, Mujoco, and Habitat environments, including task descriptions, state and action spaces, and the specific objectives of each environment.

\stitle{Habitat Environment} The Habitat platform provides a highly realistic 3D indoor simulation environment, ideal for testing the performance of algorithms in visual navigation tasks. Habitat requires agents to navigate complex 3D indoor scenes and make decisions based on multi-sensor information to complete set task objectives.

\begin{figure}[H]
\centering
\includegraphics[width=0.9\textwidth]
{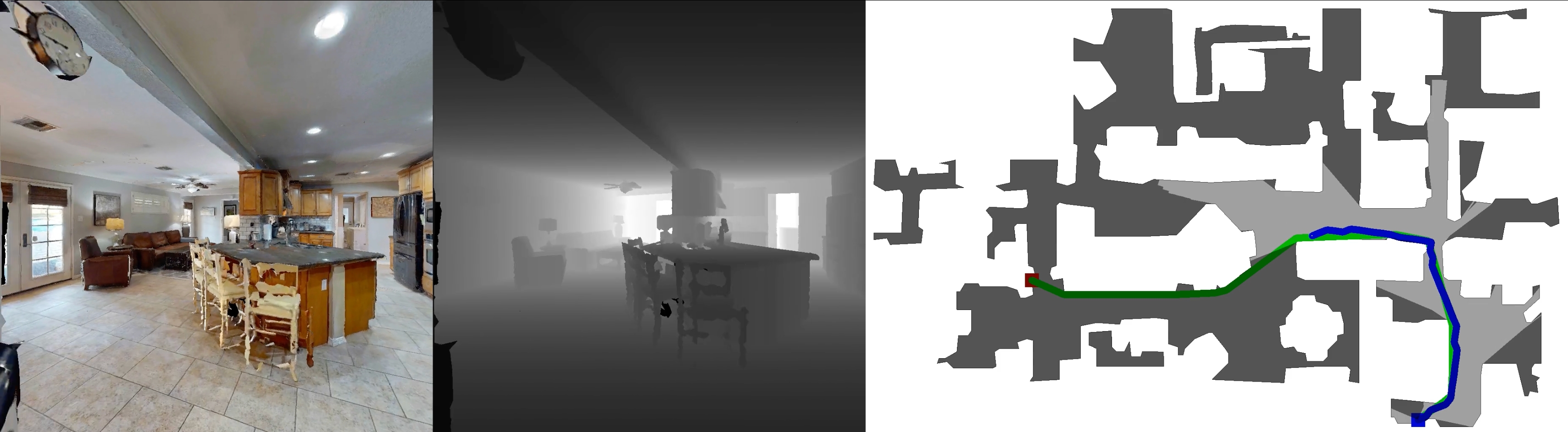}
\caption{Habitat environment for RL with sparse rewards}
\label{fig:ha}
\end{figure}

\textbf{Environment Description}

The Habitat platform is based on the HM3D (Habitat-Matterport 3D) dataset, providing high-quality indoor scene renderings that realistically recreate various residential, office, and commercial spaces. This dataset captures 3D indoor scenes from real buildings, covering a variety of different indoor layouts and structures, ensuring diversity and complexity in experimental environments. Agents in these scenes are equipped with multiple sensors, including RGB cameras, depth cameras, and semantic segmentation sensors, providing comprehensive environmental perception information to assist agents in precise navigation in three-dimensional spaces.

\textbf{Task Objectives}

We selected three typical scenarios to test EDER's performance in visual navigation tasks:

\begin{itemize}
    \item \textbf{Residential Environment}: Typical household spaces such as living rooms and bedrooms. The task requires the agent to navigate complex home environments and find target locations. This task tests the agent's navigation and decision-making abilities when faced with various obstacles and diverse layouts in a home setting.
    \item \textbf{Office Environment}: Including workspaces, meeting rooms, and corridors within office buildings. The task focuses on testing the agent's navigation ability in complex architectural structures, particularly its performance in handling long-path planning and multiple obstacles.
    \item \textbf{Commercial Environment}: Simulating scenes in shops or shopping centers, these scenarios contain open spaces and various visual elements. The task requires the agent to accurately locate targets in this dynamic and diverse environment, evaluating its adaptability and response capabilities under complex visual inputs.
\end{itemize}

\stitle{Atari Environment} The Atari game environment is a commonly used test platform for discrete action space in reinforcement learning, offering numerous challenging game scenarios for evaluating exploration strategies and complex decision-making processes. In our experiments, we selected various types of games from the Atari 57 suite to comprehensively assess the performance of the EDER method across different game categories. In the Atari environment, the agent's state is composed of pixel data from the game screen, with each frame represented as an RGB image. The action space is discrete, with each game having a set of predefined actions such as moving, jumping, or attacking. The reward signal is typically directly related to the game score, with the agent learning the optimal strategy by maximizing the score.

\begin{figure}[ht]
\centering
\includegraphics[width=0.9\textwidth]
{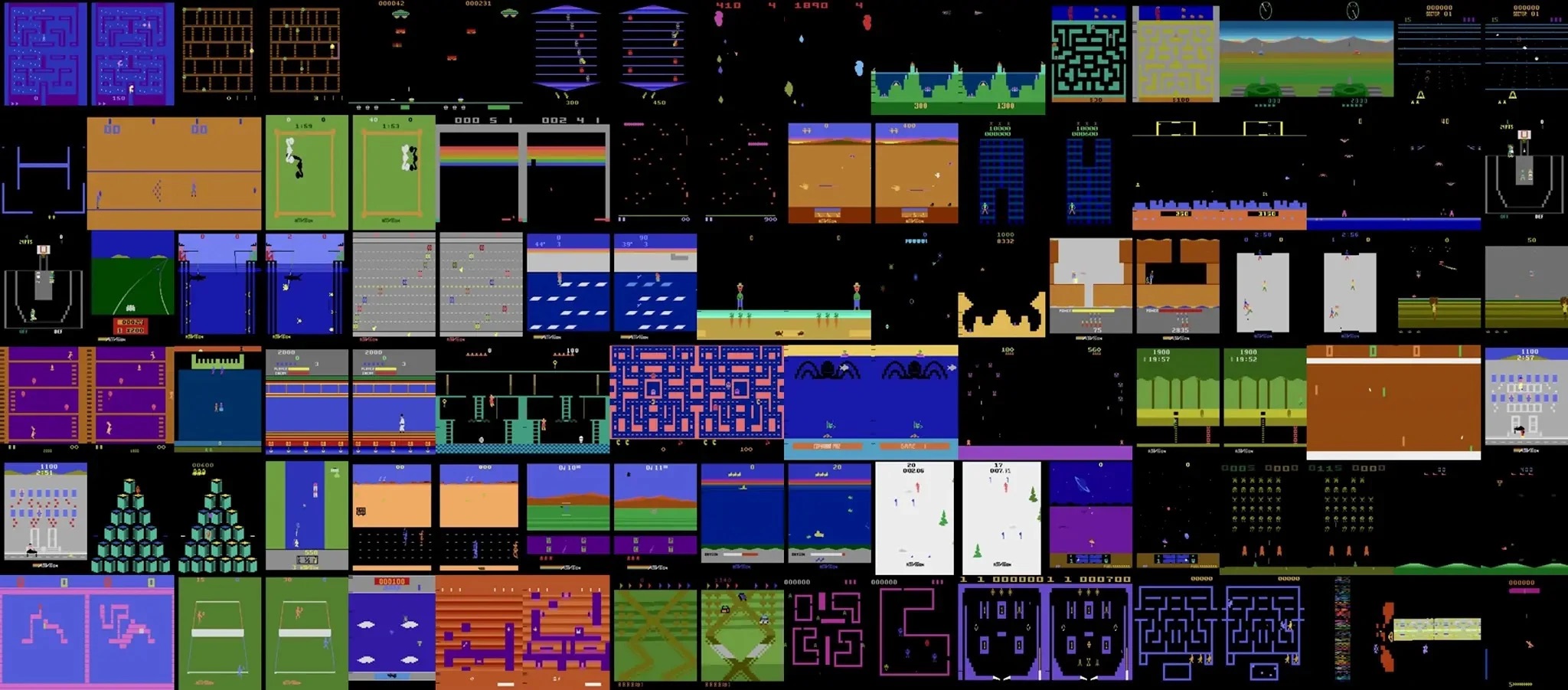}
\caption{the atari environment for RL with sparse rewards}
\label{fig:at}
\end{figure}

\textbf{Task Objectives} The experiments cover the following typical game categories to test EDER's performance under different challenges:\textbf{Exploration Games}: These games generally have sparse reward structures, requiring deep exploration by the agent to achieve high scores, making them suitable for evaluating EDER's exploration capabilities.\textbf{Action Games}: These games emphasize quick reactions and precise control, testing the agent's ability to execute strategies in highly dynamic environments.\textbf{Strategy Games}: These games require the agent to formulate and execute complex strategies, making them ideal for evaluating EDER's long-term planning and decision-making abilities. In each category, we compared the performance of EDER with other baseline methods, including average score, success rate, and convergence speed. Through these tests, we comprehensively evaluated EDER's adaptability and performance across different types of tasks, supporting its effectiveness in a wide range of reinforcement learning applications.

\stitle{Mujoco Environment}The Mujoco environment is a widely-used physics engine for robot control and continuous action space tasks. It simulates realistic physical interactions, making it ideal for testing strategies in complex manipulation tasks and fine motor control. In our experiments, we focused on two classic task environments: FetchEnv and HandEnv, representing different types of robotic manipulation tasks, and conducted multiple experiments to evaluate the performance of the EDER method in these environments.The picture can be seen in Figuare ~\ref{fig:mu}

\begin{figure}[H]
\centering
\includegraphics[width=\textwidth]
{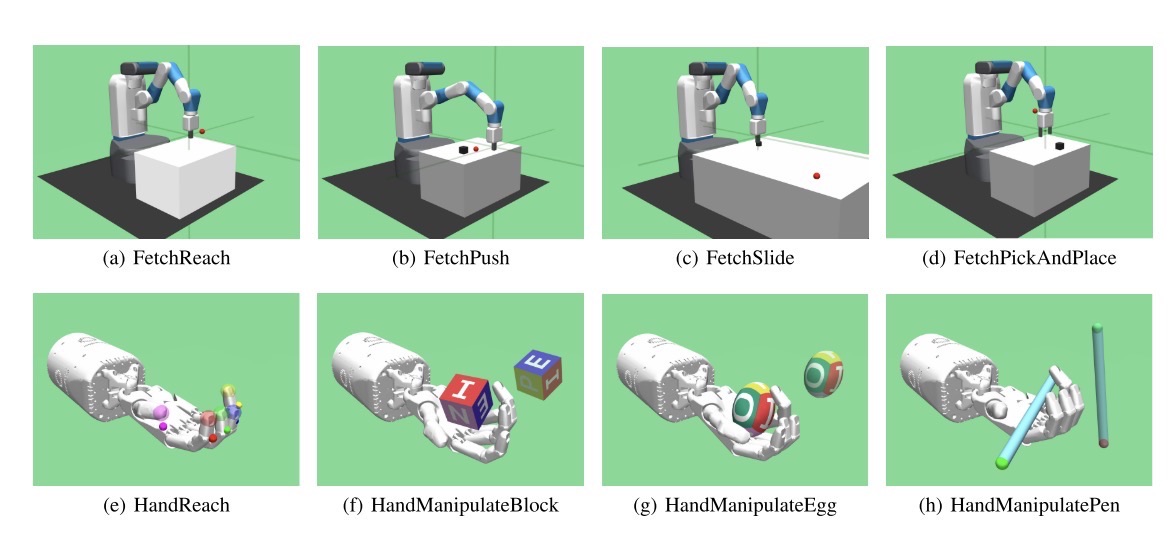}
\caption{the Open AI Robotics environment for RL with sparse rewards}
\label{fig:mu}
\end{figure}

\textbf{FetchEnv} is based on the Fetch robotic arm simulation, featuring a robotic arm with seven degrees of freedom equipped with a two-finger parallel gripper for grasping and manipulating objects. The task of the robotic arm is to move, slide, or place an object within a three-dimensional space. The action space is a four-dimensional vector, with three dimensions controlling the gripper's movement along the x, y, and z axes, and the fourth dimension controlling the gripper's opening and closing. In each step, the agent adjusts the gripper's position and orientation by selecting an action vector. Observations include the following:

\begin{itemize}
    \item Cartesian coordinates of the gripper's position (x, y, z).
    \item Linear velocity of the gripper (vx, vy, vz).
    \item Position information of the object and its relative position to the target.
\end{itemize}

\textbf{Task Objectives} FetchEnv includes multiple tasks, each with specific goals and challenges. The four main tasks tested in our experiments are:

\begin{itemize}
    \item \textbf{FetchPickAndPlace}: The agent learns to pick up objects and place them accurately at the target position.
    \item \textbf{FetchPush}: The agent must push objects along a specified path to the target area.
    \item \textbf{FetchReach}: The task requires the agent to bring the gripper as close as possible to the target object by controlling the gripper's position.
    \item \textbf{FetchSlide}: The agent slides the object to a distant target area, testing precise control and long-term planning abilities.
\end{itemize}

\textbf{HandEnv} is based on the simulation of the Shadow Dexterous Hand, which has 24 degrees of freedom, making it a highly complex robotic manipulation environment. The dexterous hand is equipped with multiple independently controlled joints, allowing for fine-grained grasping and manipulation tasks, making it an ideal platform for testing reinforcement learning algorithms in high-dimensional action spaces. The action space is a 20-dimensional vector corresponding to the positions of 20 independently controlled joints. The remaining coupled joints work in coordination with these controlled joints. In each step, the agent completes complex manipulation tasks by adjusting the postures of these joints. Observations include:

\begin{itemize}
    \item Position and velocity information of each independently controlled joint.
    \item Object position, rotation angle, and deviation from the target posture.
\end{itemize}

\textbf{Task Objectives} The tasks in HandEnv are more complex than in FetchEnv, mainly testing the agent's fine control ability in high-dimensional action spaces. The four main tasks tested in our experiments are:

\begin{itemize}
    \item \textbf{HandBlockRotate}: The agent learns to grasp a block and rotate it to a specified angle in the air.
    \item \textbf{HandEggFull}: The task requires the agent to grasp and rotate an egg to a specific position without damaging the object.
    \item \textbf{HandPenRotate}: The agent manipulates a pen, enabling it to perform complex rotations between the fingers.
    \item \textbf{HandReach}: The task requires the agent to precisely control the fingers to reach the target position accurately.
\end{itemize}

In these environments, we conducted a comprehensive evaluation of EDER, focusing on its navigation ability and strategy learning under high-dimensional visual input and complex environments. Through these experiments, we aim to validate EDER's effectiveness in high-dimensional visual tasks, particularly its potential applications in complex real-world scenarios.

\subsection{Hardware and Software Setup}

Our experiments were conducted on a server equipped with an NVIDIA Tesla V100 GPU and 32GB of memory. The operating system was Ubuntu 18.04, and the key software libraries used included PyTorch 2.4.0, OpenAI Gym 0.17.3, MuJoCo 2.0, and Habitat API 0.2.1. All experimental code was implemented in a Python 3.9 environment.

\subsection{Codebase Used}
\label{sec:codebased_used}
Our method was implemented by building on top of the following codebases:

\begin{itemize}
\item \textbf{Habitat-Lab}: \url{https://github.com/facebookresearch/habitat-lab.git}
\item \textbf{Habitat-Sim}: \url{https://github.com/facebookresearch/habitat-sim.git}
\item \textbf{TER}: \url{https://github.com/Improbable-AI/ter.git}
\item \textbf{LaBER}: \url{https://github.com/SuReLI/laber.git}
\item \textbf{Relo}: \url{https://github.com/shivakanthsujit/reducible-loss.git}
\end{itemize}

\subsection{RL Hyperparameters}

In this section, we detail the hyperparameters used in the reinforcement learning (RL) algorithms across the Atari, Mujoco, and Habitat environments. The selection of these hyperparameters has a crucial impact on the algorithm's performance, and through a series of experiments, we determined the optimal configurations to ensure effectiveness and robustness in each task.

\stitle{Atari Environment}

In the Atari environment, we used the following hyperparameter configuration for Deep Q-Network (DQN) and its variants:

\begin{table}[H]
\centering
\caption{Hyperparameter Configuration for Atari Environment}
\begin{tabular}{l c}
\toprule
Parameter & Value \\
\midrule
Policy & Default \\
Environment & Specified \\
Learning Rate & 0.0001 \\
Buffer Size & 1,000,000 \\
Learning Starts & 100 steps \\
Batch Size & 32 \\
Tau & 1.0 \\
Discount Factor ($\gamma$) & 0.99 \\
Train Frequency & 4 steps \\
Gradient Steps & 1 step \\
Target Update Interval & 10,000 steps \\
Exploration Fraction & 0.1 \\
Exploration Initial $\epsilon$ & 1.0 \\
Exploration Final $\epsilon$ & 0.05 \\
Max Gradient Norm & 10 \\
Stats Window Size & 100 \\
Verbose & 0 \\
\bottomrule
\end{tabular}
\end{table}

\stitle{Mujoco Environment}

In the Mujoco environment, we used the following hyperparameter configuration for Deep Deterministic Policy Gradient (DDPG) and its variants:

\begin{table}[H]
\centering
\caption{Hyperparameter Configuration for Mujoco Environment}
\begin{tabular}{l c}
\toprule
\textbf{Parameter} & \textbf{Value} \\
\midrule
Learning Rate (Actor) & 0.001 \\
Learning Rate (Critic) & 0.001 \\
Discount Factor ($\gamma$) & 0.99 \\
Replay Buffer Size & 1,000,000 \\
Minibatch Size & 64 \\
Target Network Update Frequency & 1,000 steps \\
$\tau$ for Soft Update & 0.005 \\
Sampling Number ($m$) & 200 \\
Trajectory Length ($b$) & 10 \\
OU-Noise ($\mu$) & 0 \\
OU-Noise ($\theta$) & 0.15 \\
OU-Noise ($\sigma$) & 0.2 \\
Optimizer & Adam \\
Adam Parameters ($\beta_1$, $\beta_2$) & (0.9, 0.999) \\
\bottomrule
\end{tabular}
\label{tab:mujoco_ddpg_config}
\end{table}

\stitle{Habitat Environment}

In the Habitat environment, we used the following hyperparameter configuration for the reinforcement learning algorithms, including DDPG, DDPG+HER, DDPG+PER, and DDPG+EDER:

\begin{table}[H]
\centering
\caption{Hyperparameter Configuration for DDPG in Habitat Environment}
\begin{tabular}{l c}
\toprule
\textbf{Parameter} & \textbf{Value} \\
\midrule
Learning Rate (Actor) & 0.001 \\
Learning Rate (Critic) & 0.001 \\
Discount Factor ($\gamma$) & 0.99 \\
Replay Buffer Size & 1,000,000 transitions \\
Batch Size & 128 \\
Rollout Length & 128 steps \\
Optimizer & Adam \\
Adam Parameters ($\beta_1$, $\beta_2$) & (0.9, 0.999) \\
Target Network Update Rate ($\tau$) & 0.005 \\
Noise Type & Gaussian noise \\
Noise Parameters (mean, std) & (0, 0.2) \\
HER Strategy & Future strategy, 4 future goals per episode \\
PER Parameters ($\alpha$, $\beta$) & (0.6, 0.4, annealed to 1.0) \\
EDER Parameters ($\lambda$, rejection prob) & (0.5, 0.1) \\
Training Timesteps & 1,000,000 steps \\
\bottomrule
\end{tabular}
\label{tab:ddpg_habitat_hyperparams}
\end{table}

These hyperparameter configurations were carefully selected to ensure optimal performance in each task environment. Through these settings, we were able to effectively compare the performance of the EDER method with other baseline methods across different environments, validating its superiority and adaptability.